\documentclass[10pt,twocolumn,letterpaper]{article}

\usepackage{cvpr}              

\renewcommand{\paragraph}[1]{\vspace{.5em}\noindent\textbf{#1.}}

\usepackage{soul}
\setuldepth{foobar}

\usepackage[utf8]{inputenc} 
\usepackage[T1]{fontenc}    
\usepackage{url}            
\usepackage{amsfonts}       
\usepackage{nicefrac}       
\usepackage{microtype}      
\usepackage{xcolor}         
\usepackage{amsmath}
\allowdisplaybreaks   
\usepackage{amsthm} 
\usepackage{mathtools}
\usepackage{mathrsfs}
\usepackage{dsfont}
\usepackage{tasks}
\usepackage{wrapfig}
\definecolor{cvprblue}{rgb}{0.21,0.49,0.74}
\usepackage[pagebackref,breaklinks,colorlinks,allcolors=cvprblue]{hyperref}
\usepackage[capitalize]{cleveref}
\usepackage{enumitem}
\usepackage{booktabs}
\usepackage{pifont}
\usepackage{makecell}
\usepackage{multicol}
\usepackage{tikz-cd}
\usepackage{units}
\usepackage[ruled,linesnumbered]{algorithm2e}
\usepackage{titletoc}
\usepackage{appendix}
\usepackage{multirow}
\usepackage{colortbl}
\usepackage{crossreftools}
\usepackage{mdframed} 
\usepackage{chngcntr} 

\newmdenv[
  backgroundcolor=gray!10, 
  linecolor=gray!10,       
  linewidth=0pt,           
  roundcorner=8pt,         
  skipabove=6pt,           
  skipbelow=6pt,           
  innertopmargin=6pt,      
  innerbottommargin=6pt,   
  leftmargin=0pt,          
  rightmargin=0pt          
]{keybox}

\usepackage{listings}
\usepackage{aliascnt}

\theoremstyle{plain}
\newtheorem{theorem}{Theorem}[section]

\newaliascnt{proposition}{theorem}
\newtheorem{proposition}[proposition]{Proposition}
\aliascntresetthe{proposition}

\newaliascnt{lemma}{theorem}

\aliascntresetthe{lemma}

\newaliascnt{corollary}{theorem}
\newtheorem{corollary}[corollary]{Corollary}
\aliascntresetthe{corollary}

\theoremstyle{definition}
\newaliascnt{definition}{theorem}
\newtheorem{definition}[definition]{Definition}
\aliascntresetthe{definition}

\newaliascnt{assumption}{theorem}

\aliascntresetthe{assumption}

\theoremstyle{remark}
\newaliascnt{remark}{theorem}
\newtheorem{remark}[remark]{Remark}
\aliascntresetthe{remark}

\providecommand{\calF}{\mathcal{F}}

\providecommand{\calO}{\mathcal{O}}
\providecommand{\calX}{\mathcal{X}}

\providecommand{\bbR}[1]{\mathbb {R}^{#1}}

\providecommand{\bbRscalar}{\mathbb {R}}

\providecommand{\dist}{\operatorname{d}}
\providecommand{\rieexp}{\operatorname{Exp}}
\providecommand{\rielog}{\operatorname{Log}}

\providecommand{\pt}[2]{\operatorname{PT}_{#1 \rightarrow #2}}

\providecommand{\rieLog}[1]{\operatorname{Log}_{#1}}

\providecommand{\softmax}{{\mathrm{softmax}}}

\providecommand{\sign}{\operatorname{sign}}

\providecommand{\id}{\mathds{1}}

\providecommand{\gyr}{\operatorname{gyr}}

\providecommand{\inner}[2]{\left\langle #1,#2 \right\rangle}

\providecommand{\norm}[1]{\left\| #1 \right\|}

\providecommand{\Rzero}{\mathbf{0}}

\providecommand{\calM}{\mathcal{M}}

\providecommand{\vecone}{\boldsymbol{1}}

\providecommand{\hyperspace}[1]{\mathcal{H}^{#1}_{K}}
\providecommand{\Hoplus}{\oplus _\mathcal{H}}
\providecommand{\Hominus}{\ominus _\mathcal{H}}
\providecommand{\Hodot}{\odot _\mathcal{H}}

\providecommand{\pball}[1]{\mathbb{P}^{#1}_{K}}
\providecommand{\Moplus}{\oplus _\mathrm{M}}
\providecommand{\Mominus}{\ominus _\mathrm{M}}
\providecommand{\Modot}{\odot _\mathrm{M}}

\providecommand{\lorentz}[1]{\mathbb{L}^{#1}_K}
\providecommand{\unitlorentz}[1]{\mathbb{L}^{#1}}
\providecommand{\Lnorm}[1]{\left\| #1 \right\|_{\mathcal{L}}}
\providecommand{\Linner}[2]{\left\langle #1, #2 \right\rangle_{\mathcal{L}}}
\providecommand{\Lzero}{\overline{\mathbf{0}}}

\providecommand{\Loplus}{\oplus _\mathbb{L}}
\providecommand{\Lominus}{\ominus _\mathbb{L}}
\providecommand{\Lodot}{\odot _\mathbb{L}}

\providecommand{\unitsphere}[1]{\mathbb{S}^{#1}}

\providecommand{\unitsphere}[1]{\mathbb{S}^{#1}_{-1}}

\providecommand{\norm}[1]{\left\| #1 \right\|}

\providecommand{\CAT}{\mathrm{CAT}}

\newcommand{\cmark}{\textcolor{green}{\text{\ding{51}}}}%
\newcommand{\xmark}{\textcolor{red}{\text{\ding{55}}}}

\providecommand{\ie}{\emph{i.e.},\xspace}
\providecommand{\eg}{\emph{e.g.},\xspace}
\providecommand{\redbf}[1]{\textbf{\textcolor{red}{#1}}}

\providecommand{\greenbf}[1]{\textbf{\textcolor{green}{#1}}}

\definecolor{HilightColor}{RGB}{240, 240, 250} 

\crefname{equation}{Eq.}{Eqs.}
\Crefname{equation}{Equation}{Equations}
\crefname{figure}{Fig.}{Figs.}
\Crefname{figure}{Figure}{Figures}
\crefname{table}{Tab.}{Tabs.}
\Crefname{table}{Table}{Tables}
\crefname{algocf}{Alg.}{Algs.}
\Crefname{algocf}{Algorithm}{Algorithms}
\crefname{section}{Sec.}{Secs.}
\Crefname{section}{Section}{Sections}
\crefname{appendix}{App.}{Apps.}
\Crefname{appendix}{Appendix}{Appendices}

\crefname{theorem}{Thm.}{Thms.}
\Crefname{theorem}{Theorem}{Theorems}
\crefname{lemma}{Lem.}{Lems.}
\Crefname{lemma}{Lemma}{Lemmas}
\crefname{definition}{Def.}{Defs.}
\Crefname{definition}{Definition}{Definitions}
\crefname{corollary}{Cor.}{Cors.}
\Crefname{corollary}{Corollary}{Corollaries}
\crefname{remark}{Rmk.}{Rmks.}
\Crefname{remark}{Remark}{Remarks}
\crefname{proposition}{Prop.}{Props.}
\Crefname{proposition}{Proposition}{Propositions}
\crefname{proof}{Pr.}{Prs.}
\Crefname{proof}{Proof}{Proofs}

\makeatletter

\newcommand{\eqnrefs}[1]{%
  \begingroup
    \edef\eqn@list{#1}%
    \count@=0\relax
    \@for\eqn@tmp:=\eqn@list\do{\advance\count@ by 1\relax}%
    \ifnum\count@=1
      Equation~\ref{#1}%
    \else
      Equations~%
      \count256=0\relax
      \@for\eqn@tmp:=\eqn@list\do{%
        \advance\count256 by 1\relax
        \ifnum\count256<\numexpr\count@-0\relax
          \ref{\eqn@tmp}%
          \ifnum\count256<\numexpr\count@-1\relax
            ,\ 
          \else
            \ 
          \fi
        \else
          and~\ref{\eqn@tmp}%
        \fi
      }%
    \fi
  \endgroup
}
\makeatother

\usepackage[acronym, toc, automake, hyperfirst=true]{glossaries-extra}

\newglossary*{notation}{List of Symbols}
\setglossarystyle{long} 
\setabbreviationstyle[acronym]{long-short} 

\makeglossaries 

\newacronym[sort=nn]{CNNs}{CNNs}{Convolutional Neural Networks}
\newacronym[sort=nn]{RNNs}{RNNs}{Recurrent Neural Networks}
\newacronym[sort=nn]{DNNs}{DNNs}{Deep Neural Networks}
\newacronym[sort=nn]{GCN}{GCN}{Graph Convolutional Network}

\newacronym[sort=nn]{FC}{FC}{Fully Connected}
\newacronym[sort=nn]{MLR}{MLR}{Multinomial Logistics Regression}
\newacronym[sort=nn]{BMLR}{BMLR}{Busemann Multinomial Logistics Regression}
\newacronym[sort=nn]{BFC}{BFC}{Busemann Fully Connected}

\newacronym[sort=nn]{SPDNet}{SPDNet}{SPD Neural Network}
\newacronym[sort=nn]{GrNet}{GrNet}{Grassmann network}
\newacronym[sort=nn]{HNN}{HNN}{Hyperbolic Neural Network}
\newacronym[sort=nn]{HNN++}{HNN++}{Hyperbolic Neural Network++}

\newacronym[sort=bn]{BN}{BN}{Batch Normalization}
\newacronym[sort=bn]{RBN}{RBN}{Riemannian Batch Normalization}
\newacronym[sort=bn]{LieBN}{LieBN}{Lie Group Batch Normalization}
\newacronym[sort=bn]{SPDBN}{SPDBN}{SPD Batch Normalization}
\newacronym[sort=bn]{GyroBN}{GyroBN}{Gyrogroup Batch Normalization}

\newacronym[sort=spd]{SPD}{SPD}{Symmetric Positive Definite}
\newacronym[sort=spd]{AIM}{AIM}{Affine-Invariant Metric}
\newacronym[sort=spd]{LCM}{LCM}{Log-Cholesky Metric}
\newacronym[sort=spd]{LEM}{LEM}{Log-Euclidean Metric}

\newacronym[sort=gr]{ONB}{ONB}{Orthonormal Basis}
\newacronym[sort=gr]{PP}{PP}{Projector Perspective}

\newacronym[sort=cor]{ECM}{ECM}{Euclidean--Cholesky Metric}
\newacronym[sort=cor]{LECM}{LECM}{Log-Euclidean--Cholesky Metric}
\newacronym[sort=cor]{LSM}{LSM}{Log-Scaled Metric}
\newacronym[sort=cor]{OLM}{OLM}{Off-Log Metric}
\newacronym[sort=cor]{PHCM}{PHCM}{Poly-Hyperbolic-Cholesky Metric}

\newacronym[sort=score]{MCC}{MCC}{Matthews Correlation Coefficient}

\newglossaryentry{manifold}{
  type=notation, 
  name={$\mathcal{M}$}, 
  text={$\mathcal{M}$}, 
  description={Riemannian manifold}, 
}

\newglossaryentry{metric}{
  type=notation,
  name={$\mathbf{g}$},
  text={$\mathbf{g}$},
  description={Metric on $\mathcal{M}$},
}

\newglossaryentry{reals}{
  type=notation,
  name={$\mathbb{R}^n$},
  text={$\mathbb{R}^n$},
  description={A $n$-dimensional Euclidean space},
}

\newcommand{\glsfull}[1]{\glsentrylong{#1} (\glsentryshort{#1})}

\def\paperID{***} 
\def\confName{CVPR}
\def\confYear{2026}

\title{Hyperbolic Busemann Neural Networks}

\author{
Ziheng Chen$^{1,2}$,
Bernhard Schölkopf$^{2}$ \&
Nicu Sebe$^{1}$ 
\\
$^1$ University of Trento, $^2$ MPI-IS \\
}

\begin{document}
\maketitle
\begin{abstract}
Hyperbolic spaces provide a natural geometry for representing hierarchical and tree-structured data due to their exponential volume growth. To leverage these benefits, neural networks require intrinsic and efficient components that operate directly in hyperbolic space. In this work, we lift two core components of neural networks, Multinomial Logistic Regression (MLR) and Fully Connected (FC) layers, into hyperbolic space via Busemann functions, resulting in Busemann MLR (BMLR) and Busemann FC (BFC) layers with a unified mathematical interpretation. BMLR provides compact parameters, a point-to-horosphere distance interpretation, batch-efficient computation, and a Euclidean limit, while BFC generalizes FC and activation layers with comparable complexity. Experiments on image classification, genome sequence learning, node classification, and link prediction demonstrate improvements in effectiveness and efficiency over prior hyperbolic layers. The code is available at \url{https://github.com/GitZH-Chen/HBNN}.
\end{abstract}    
\section{Introduction}
\label{sec:intro}

Hyperbolic representations have recently delivered strong performance across different applications because the exponential volume growth of negatively curved manifolds enables low-distortion embeddings of tree-like and hierarchical structure \citep{nickel2017poincare}. These advantages have been validated in computer vision \citep{gao2021curvature,khrulkov2020hyperbolic,ermolov2022hyperbolic,van2023poincare,gao2023exploring,bdeir2024fully,he2025lorentzian,bdeir2025robust,sur2025hyperbolic,liu2025hyperbolic,wang2026wasserstein}, graph learning \citep{chami2019hyperbolic,bachmann2020constant,fu2024hyperbolic,sun2024geometry}, multimodal learning \citep{desai2023hyperbolic,pal2025compositional}, recommendation systems \citep{yanghg2025former}, astronomy \citep{chen2025galaxy}, genome sequence learning \citep{khan2025hyperbolic}, natural language processing \citep{nickel2017poincare,ganea2018hyperbolic,nickel2018learning,gulcehre2019hyperbolic,he2025helm,yang2025hyperbolic}, and brain signal decoding \citep{li2026heegnet}. These empirical successes highlight the need for principled and efficient hyperbolic neural network components that fully leverage hyperbolic geometry.

Among canonical formulations, the Poincaré ball and the Lorentz model are the most widely used, owing to their closed-form Riemannian operators. To support hyperbolic deep learning, several key building blocks in neural networks have been generalized to Poincaré or Lorentz spaces, including attention \citep{gulcehre2019hyperbolic,chen2022fully,yang2024hypformer}, batch normalization \citep{lou2020differentiating,bdeir2024fully,chen2025gyrobn,chen2025gyrobnextension}, linear feed-forward layers \citep{ganea2018hyperbolic,shimizu2021hyperbolic,chen2022fully}, activation \citep{ganea2018hyperbolic,bdeir2024fully}, residual blocks \citep{van2023poincare,katsman2024riemannian,he2025lorentzian}, \glsfull{MLR} \citep{shimizu2021hyperbolic,bdeir2024fully,nguyen2025neural}, and graph convolution \citep{chami2019hyperbolic,liu2019hyperbolic,bachmann2020constant,dai2021hyperbolic}. 
Among these components, MLR classification and \glsfull{FC} layers play a fundamental role in feature transformation and final decision-making.

Recently, hyperplanes and point-to-hyperplane distances \citep{lebanon2004hyperplane} have been adopted to construct hyperbolic MLR in both Poincaré \citep{ganea2018hyperbolic,shimizu2021hyperbolic,nguyen2025neural} and Lorentz \citep{bdeir2024fully} models. \citet[Sec.~3.1]{ganea2018hyperbolic} introduced the first Poincaré MLR based on Poincaré hyperplanes, but the formulation suffers from over-parameterization and lacks batch efficiency. \citet[Sec.~3.1]{shimizu2021hyperbolic} alleviated these issues through a re-parameterization strategy. Building on these ideas, \citet[Sec.~4.3]{bdeir2024fully} proposed a Lorentz MLR. However, its hyperplane is defined by the ambient Minkowski space, which is model-specific and may distort Lorentzian geometry.

For hyperbolic FC layers, three main formulations exist. \citet[Sec.~3.2]{ganea2018hyperbolic} introduced Möbius matrix-vector multiplication through the tangent space on the Poincaré ball. \citet[Sec.~3.2]{shimizu2021hyperbolic} further proposed the Poincaré FC layer, defined intrinsically but restricted to the Poincaré model. On the Lorentz model, \citet[Sec.~2.2]{chen2022fully} constructed a Lorentz FC layer by applying linear transformations in the ambient Minkowski space followed by projection onto the Lorentz model. Thus, Möbius and Lorentz FC rely on flat-space (tangent or ambient) approximations that could distort intrinsic geometry, whereas Poincaré FC is intrinsic but model-specific.

On the other hand, the Busemann function and its level sets, horospheres, have emerged as powerful intrinsic tools for hyperbolic learning. They enjoy convenient metric properties \citep[Ch. II.8]{bridson2013metric} and admit closed-form expressions on both the Poincaré and Lorentz models \citep[Prop.~9]{bonet2025sliced}. These operators have supported several hyperbolic algorithms, including SVM \citep{fan2023horospherical}, PCA \citep{chami2021horopca}, Sliced Wasserstein distances \citep{bonet2025sliced}, and prototype learning \citep{ghadimi2021hyperbolic}. We also note that \citet[Cor.~4.3]{nguyen2025neural} proposed a Poincaré MLR based on the Busemann function. However, its induced point-to-hyperplane distance is pseudo, coincides with the true distance only in Euclidean geometry, remains over-parameterized, and is not batch efficient.

These observations motivate intrinsic and batch-efficient formulations for MLR and FC layers that can operate on both the Poincaré ball and the Lorentz model. To address this need, we propose \emph{\glsfull{BMLR}} and \emph{\glsfull{BFC}} layers, two Busemann-based components for hyperbolic networks. Our contributions are summarized as follows:
\begin{itemize}

\item
We introduce BMLR, deriving intrinsic logits directly from Busemann functions with a point-to-horosphere distance interpretation. BMLR uses a compact per-class parameterization, eliminates manifold-valued parameters in prior MLRs, remains batch-efficient, and recovers Euclidean MLR as curvature tends to zero.

\item
We develop BFC layers by generalizing the FC and activation layers through the Busemann function, providing intrinsic constructions on both the Poincaré and Lorentz models. BFC preserves comparable complexity and parameter counts, and recovers Euclidean FC in the zero curvature limit.

\item
We provide empirical validation across image classification, genome sequence learning, node classification, and link prediction. BMLR and BFC generally outperform existing hyperbolic layers. BMLR shows particularly large gains as the number of classes increases, and the Lorentz BMLR is the fastest among all hyperbolic MLRs.
\end{itemize}

\section{Preliminaries}
\label{sec:preliminaries}

\textbf{Hyperbolic geometry.}
An $n$-dimensional hyperbolic space is a Riemannian manifold with constant negative sectional curvature $K<0$. There exist five models of hyperbolic space \citep[Sec.~7]{cannon1997hyperbolic}. Among them, the Poincaré ball and Lorentz models are the most widely used. The \emph{Poincaré ball model} \citep[p.~62]{lee2018introduction} is
$\pball{n}=\{x\in \bbR{n}: \norm{x}^2<-1/K \}$ with Riemannian metric $g_x(u,v)=(\lambda_x^K)^2\inner{u}{v}$, where $\lambda_x^K=\frac{2}{1+K\norm{x}^2}$. Here $\inner{\cdot}{\cdot}$ and $\norm{\cdot}$ denote the Euclidean inner product and norm. The origin is the zero vector $\Rzero \in \pball{n}$. The \emph{Lorentz model} \citep[p.~62]{lee2018introduction}, also called the hyperboloid model, is
$\lorentz{n}=\{ x\in \bbR{n+1}: \Linner{x}{x}=1/K,\ x_t>0 \}$, where $\Linner{x}{y}=-x_t y_t+\inner{x_s}{y_s}$ is the Lorentzian inner product and $\Lnorm{x}=\sqrt{\Linner{x}{x}}$ is the induced Lorentzian norm. Following the convention of special relativity \citep[§3.1]{ratcliffe2006foundations}, we write $x=[x_t, x_s^\top]^\top$ with temporal component $x_t\in\bbRscalar$ and spatial component $x_s\in\bbR{n}$. The canonical origin is $\Lzero=[\sqrt{1/|K|},0,\ldots,0]^\top \in \lorentz{n}$. Throughout, we write $\hyperspace{n}\in\{\pball{n},\lorentz{n}\}$ for a hyperbolic space.

\textbf{Geodesic, geodesic ray, and asymptote.}
\emph{Geodesics} are length-minimizing curves. We consider unit speed geodesics. A \emph{geodesic ray} is a geodesic defined for $t \in [0,\infty)$. Two geodesic rays are \emph{asymptotic} if the distance between corresponding points remains bounded as $t$ tends to infinity \citep[II. Def.~8.1]{bridson2013metric}. This notion generalizes Euclidean parallel lines in two respects: (i) Euclidean parallels have constant separation, whereas asymptotic rays have bounded separation; (ii) Euclidean parallels never intersect, and asymptotic rays do not intersect as well \citep[II. Prop.~8.2]{bridson2013metric}.

\textbf{Busemann function.}
Let $\gamma$ be a geodesic ray and let $\dist(\cdot,\cdot)$ denote the geodesic distance. The \emph{Busemann function} \citep[II. Def.~8.17]{bridson2013metric} associated with $\gamma$ is defined as
\begin{equation}
\forall x \in \hyperspace{n}, \quad
B^{\gamma}(x) = \lim_{t\to\infty} \left( \dist\left(x,\gamma(t)\right) - t \right).
\end{equation}
In hyperbolic spaces, this limit exists \citep[II. Lem.~8.18]{bridson2013metric}. In Euclidean space, the Busemann function associated with the geodesic $\gamma(t)=t v$ that starts at $\Rzero$ with unit direction $v \in \unitsphere{n-1} =\{ v \in \bbR{n} \mid \norm{v}=1 \}$ is
\begin{equation}
    \forall x \in \bbR{n}, \quad B^v(x) = -\inner{x}{v},
\end{equation}
which coincides, up to a sign, with the inner product. Hence, the Busemann function provides an intrinsic generalization of the inner product to manifolds. Up to an additive constant, it is equivalent among asymptotic geodesic rays \citep[II. Cor.~8.20]{bridson2013metric}, independent of the starting point. Consequently, we write $B^v(x)$ for the ray that emanates from the origin $e \in \hyperspace{n}$ with unit velocity $v \in \unitsphere{n-1} \subset T_{e}\hyperspace{n}$. For $v \in \unitsphere{n-1}$ and $x \in \hyperspace{n}$, closed forms of the Poincaré and Lorentz Busemann functions \citep[Prop.~9]{bonet2025sliced} are
\begin{align}
    \label{eq:poincare-busemann}
    &\pball{n}: \ B^{v}(x)
    = \frac{1}{\sqrt{-K}} \log\left( 
    \frac{\norm{v - \sqrt{-K} x}^{2}}{1 + K\norm{x}^{2}} 
    \right), \\
    \label{eq:lorentz-busemann}
    &\lorentz{n}: \ B^{v}(x)
    = \frac{1}{\sqrt{-K}} \log\left(
    \sqrt{-K} \left(x_t - \inner{x_s}{v}\right)
    \right).
\end{align}

\begin{figure}[t!]
\centering
\includegraphics[width=\linewidth,trim=0 0cm 0 0]{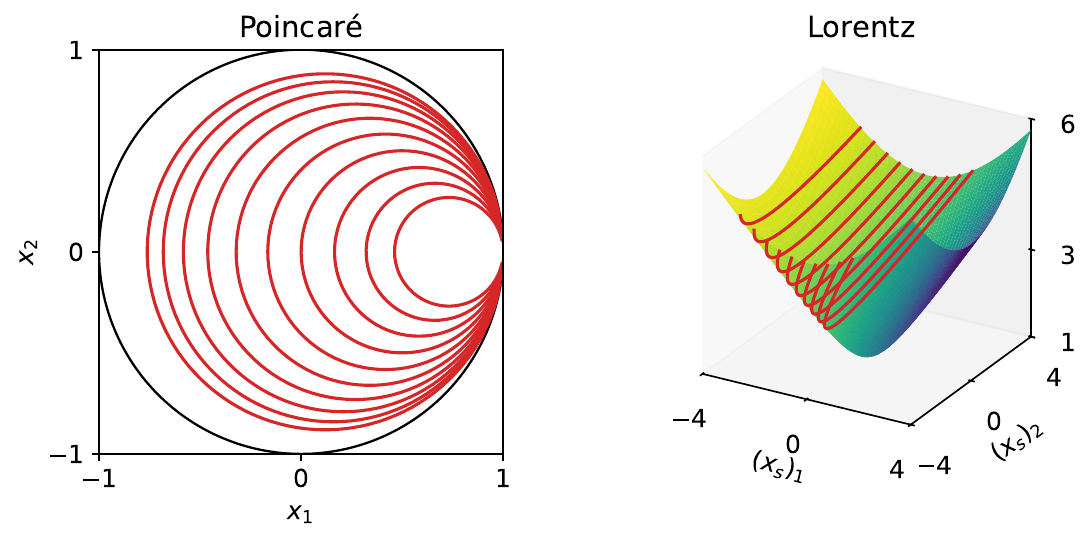}
\caption{Illustration: red curves are different horospheres of $B^{v}$.}
\label{fig:horospheres}
\vspace{-5mm}
\end{figure}

\textbf{Horosphere.}
The level sets of a Busemann function are called \emph{horospheres}, the hyperbolic counterpart of Euclidean hyperplanes. In Euclidean space, for a unit direction $v$, the hyperplanes $H^{v}_{\tau}=\{ x \in \bbR{n} \mid \inner{x}{v}=\tau \}$ with $\tau \in \bbRscalar$ are parallel. Analogously, for fixed $v$, the horospheres $H^{v}_{\tau}=\{ x \in \hyperspace{n} \mid B^{v}(x)=\tau \}$ are equidistant (see \cref{thm:distance-horospheres}). \cref{fig:horospheres} illustrates such horospheres, and \cref{tab:euclid-hyperbolic} summarizes the correspondence of the above concepts between Euclidean and hyperbolic geometries.

\begin{table}[t]
\centering
\caption{Generalization: Euclidean vs. hyperbolic.}
\label{tab:euclid-hyperbolic}
\resizebox{\linewidth}{!}{
\begin{tabular}{cc}
\toprule
Euclidean & Hyperbolic \\
\midrule
Straight line & Geodesic \\
Parallel lines & Asymptotic geodesic rays \\
Inner product $\inner{v}{x}$ & Busemann function $-B^{v}(x)$ \\
Hyperplane & Horosphere \\
Parallel hyperplanes with $v \in \unitsphere{n-1}$ & Horospheres with $v \in \unitsphere{n-1} \cap T_{e}\hyperspace{n}$ \\
\bottomrule
\end{tabular}
}
\end{table}

\textbf{Riemannian \& gyrovector operators.}
The hyperbolic space $\hyperspace{n}$ provides closed-form expressions for standard Riemannian operators. It also supports a gyrovector structure that extends vector space operations to hyperbolic geometry. On the Poincaré ball $\pball{n}$, this structure corresponds to the Möbius gyrovector space \citep[Ch.~6.14]{ungar2022analytic}, denoted as $\{\pball{n}, \Moplus, \Modot\}$. As shown by \citet[Props.~24 and 25]{chen2025gyrobnextension}, the Lorentz model also admits a closed-form gyrovector space, denoted as $\{\lorentz{n}, \Loplus, \Lodot\}$. Please refer to \cref{app:subsec:hyperbolic-geometry} for detailed expressions.

\section{Busemann multinomial logistic regression}
We begin by reformulating the Euclidean \glsfull{MLR}, then lift it to hyperbolic space via the Busemann function, introducing \emph{Busemann MLR (BMLR)}. We also present a point-to-hyperplane interpretation. Finally, we compare BMLR with existing hyperbolic MLRs, highlighting our advantages in geometric fidelity, parameterization, and computational efficiency.

\subsection{Formulation}
The Euclidean MLR $\softmax(Ax+b)$ computes the multinomial probability for each class $k \in \{1,\ldots, C\}$ given an input $x \in \bbR{n}$. It admits the inner product form:
\begin{equation} \label{eq:euc-mlr}
\forall k,\ p(y=k\mid x) = \frac{\exp(\inner{a_k}{x} + b_k)}{\sum_{j=1}^C \exp(\inner{a_j}{x} + b_j)},
\end{equation}
where $a_k \in \bbR{n}$ and $b_k \in \bbRscalar$ are the weight and bias for the class $k$. We write $p(y=k \mid x) \propto \exp\left(u_k(x)\right)$ with $u_k(x) = \inner{a_k}{x} + b_k$. Decomposing the weight vector into a magnitude $\alpha_k = \norm{a_k} > 0$ and a unit direction $v_k = \frac{a_k}{\norm{a_k}} \in \unitsphere{n-1}$, each logit is
\begin{equation}
\label{eq:vk-decomp}
u_k(x)= \alpha_k \inner{v_k}{x} + b_k.
\end{equation}

As shown in \cref{tab:euclid-hyperbolic}, the Busemann function naturally generalizes the Euclidean inner product. Analogously to \cref{eq:vk-decomp}, we define the hyperbolic logits via the Busemann function, yielding BMLR:
\begin{align}
\forall k,\ p(y=k\mid x) 
&= \frac{\exp(u_k(x))}{\sum_{j=1}^C \exp(u_j(x))}, \\
\label{eq:b-logits}
u_k(x) &= -\alpha_k B^{v_k}(x) + b_k,
\end{align}
with $\alpha_k > 0$, $v_k \in \unitsphere{n-1}$, and $b_k \in \bbRscalar$ as parameters.

The following results show that, as $K \to 0^{-}$, both the Poincaré and Lorentz BMLRs reduce to the Euclidean MLR.

\begin{theorem}[Limits of BMLRs] \label{thm:limits-bmlr}
As $K\to 0^{-}$, the hyperbolic Busemann functions converge to the Euclidean inner product:
\begin{align}
\text{(Poincaré)} \quad & B^{v}(x) \xrightarrow{K\to 0^{-}} -2\inner{v}{x}, \\
\label{eq:lorentz-b-limits}
\text{(Lorentz)} \quad & B^{v}(x) \xrightarrow{K\to 0^{-}} -\inner{v}{x_s}.
\end{align}
The hyperbolic BMLRs converge to the Euclidean MLR:
\begin{align}
\text{(Poincaré)} \quad &u_k(x) \xrightarrow{K\to 0^{-}} 2\alpha_k\inner{v_k}{x} + b_k, \\
\text{(Lorentz)} \quad &u_k(x) \xrightarrow{K\to 0^{-}} \alpha_k\inner{v_k}{x_s} + b_k,
\end{align}
\end{theorem}
\begin{proof}
    The proof is provided in \cref{app:prf:thm:limits-bmlr}.
\end{proof}

\begin{remark}[Intuition]
On the Poincaré ball, letting $K \to 0^{-}$ recovers Euclidean geometry \citep[App.~A.4.2]{skopek2020mixed}. For the Lorentz model, as $K \to 0^{-}$, the temporal coordinate diverges while the spatial component approaches $\bbR{n}$, making $\lorentz{n}$ converge to a Euclidean space. Consistently, the Poincaré and Lorentz Busemann functions and the associated BMLR logits reduce to their Euclidean counterparts, providing a natural generalization of Euclidean MLR.
\end{remark}

\begin{table*}[t]
    \centering
    \caption{Comparison of $C$-class MLR. In Dist, \greenbf{Real} means the point-to-hyperplane distance is the real distance, obtained by $\inf_{y \in H} \dist(x,y)$, where $H$ is a hyperplane and $\dist$ is the geodesic distance; \redbf{Pseudo} denotes a surrogate that coincides with the real distance only in Euclidean geometry. Compact params indicate whether each logit avoids an additional manifold-valued parameter. Batch efficiency indicates whether the MLR can avoid inefficient per-class loops in implementation (see \cref{app:sec:mlr-comparison}). In \#Params, we highlight the heaviest in \redbf{red}. In FLOPs, we mark the slowest in \redbf{red} and the fastest in \greenbf{green}.}
    \label{tab:logit-comparison}
    \resizebox{\linewidth}{!}{
    \begin{tabular}{cccccccc}
        \toprule
        Method & Logit $u_k(x), \ \forall k \in \{1,\ldots, C\}$ & Space & Dist & \#Params & \makecell{Compact \\ params} & FLOPs & \makecell{Batch \\ efficiency}\\
        \midrule
        Euclidean MLR & $\inner{a_k}{x} + b_k$, with $a_k \in \bbR{n}, b_k \in \bbRscalar$ & $\bbR{n}$ & \greenbf{Real} & $C(n+1)$ & \cmark & $C(2n)$ & \cmark \\
        \midrule
        \makecell{Poincaré MLR \\ \citep[Eq. (25)]{ganea2018hyperbolic}} & \makecell{$\displaystyle \frac{\lambda_{p_k}^{K} \norm{a_k}}{\sqrt{-K}} \sinh^{-1}\left(\frac{2\sqrt{-K} \inner{-p_k \Moplus x}{a_k}}{\left(1 + K \norm{-p_k \Moplus x}^{2}\right) \norm{a_k}}\right)$, \\ with $p_k \in \pball{n}, a_k \in T_{p_k} \pball{n}$} & $\pball{n}$ & \greenbf{Real} & $C(2n)$ & \xmark & $C(19n+29)$ & \xmark \\
        \midrule
        \makecell{Poincaré MLR \\ \citep[Eq. (6)]{shimizu2021hyperbolic}} & \makecell{$\displaystyle \frac{2}{\sqrt{-K}} \alpha_k \sinh^{-1}\left( \alpha - \beta \right)$, \\ $\alpha = \lambda_{x}^{K} \sqrt{-K} \inner{x}{v_k} \cosh\left(2\sqrt{-K} b_k\right)$, \\ $\beta = \left(\lambda_{x}^{K} - 1\right) \sinh\left(2\sqrt{-K} b_k\right)$, \\ with $\alpha_k > 0, v_k \in \unitsphere{n-1}, b_k \in \bbRscalar$} & $\pball{n}$ & \greenbf{Real} & $C(n+2)$ & \cmark & $C(4n+52)$ & \cmark \\
        \midrule
        \makecell{Pseudo-Busemann MLR \\ \citep[Cor. 4.3]{nguyen2025neural}} & \makecell{$\displaystyle - d(x, p_k) \frac{B^{v_k}\left(- p_k \Moplus x\right)}{\norm{- p_k \Moplus x}}$, \\ with $p_k \in \pball{n}, v_k \in \unitsphere{n-1}$} & $\pball{n}$ & \redbf{Pseudo} & \redbf{$C(2n)$} & \xmark & \redbf{$C(19n+34)$} & \xmark \\
        \midrule
        \makecell{ Lorentz MLR \\ \citep[Eq. (12)]{bdeir2024fully}} & 
        \makecell{$\displaystyle \frac{1}{\sqrt{-K}} \sign(\alpha) \beta\left|\sinh ^{-1}\left(\sqrt{-K} \frac{\alpha}{\beta}\right)\right|$, \\
        $\alpha=\cosh\left(\sqrt{-K} b_k\right)\inner{z_k}{x_s}-\sinh\left(\sqrt{-K} b_k\right)$, \\
        $\beta=\sqrt{ \norm{\cosh (\sqrt{-K} b_k) z_k }^2-(\sinh (\sqrt{-K} b_k)\norm{z_k})^2}$, \\ with $z_k \in \bbR{n},\ b_k \in \bbRscalar$}
        & $\lorentz{n}$ & \greenbf{Real} & $C(n+1)$ & \cmark & $C(4n+52)$ & \cmark \\
        \midrule
        \rowcolor{HilightColor} BMLR & \Gape[0pt][2pt]{\makecell{$-\alpha_k B^{v_k}(x) + b_k$, \\ with $\alpha_k > 0, v_k \in \unitsphere{n-1}, b_k \in \bbRscalar$}} & \Gape[0pt][2pt]{\makecell{$\pball{n}$ \\ $\lorentz{n}$}} & \greenbf{Real} & $C(n+2)$ & \cmark & \makecell{ $\pball{n}: C(6n+12)$ \\ \greenbf{$\lorentz{n}: C(2n+12)$} } & \cmark \\
        \bottomrule
    \end{tabular}
    }
\end{table*}

\subsection{Geometric interpretation}
As shown by \citet[Sec. 5]{lebanon2004hyperplane}, the Euclidean MLR can be reformulated by point-to-hyperplane distances:
\begin{equation}\label{eq:euc-point-to-hyperplane} 
    u_k(x) = \sign\left(\inner{a_k}{x} + b_k\right) \norm{a_k} 
    \dist\left(x, H_{a_k,b_k}\right),
\end{equation}
with $a_k \in \bbR{n}$ and $b_k \in \bbRscalar$. Here, $H_{a_k,b_k} = \{x \in \bbR{n}: \inner{a_k}{x} + b_k = 0\}$ denotes the margin hyperplane, and the point-to-hyperplane distance is 
\( \dist\left(x, H_{a_k,b_k}\right) = \frac{\left|\inner{a_k}{x} + b_k\right|}{\norm{a_k}} \). We next show that our BMLR also admits such an interpretation through point-to-horosphere distances.

We first introduce these distances in Hadamard spaces, metric spaces with favorable properties (see \cref{app:def:hadamard}), where Euclidean and hyperbolic geometries arise as special cases.

\begin{theorem}[Hadamard horosphere distance]\label{thm:distance-horospheres}
Let $(\calX,\dist)$ be a Hadamard space, and $B^\gamma: \calX \to \bbRscalar$ be the Busemann function associated with a geodesic ray $\gamma:[0,\infty)\to\calX$. For any $\tau_1,\tau_2 \in \bbRscalar$, define the horospheres by
\begin{equation}
    H^\gamma_{\tau_i}=\left\{ x \in \calX \mid B^\gamma(x)=\tau_i \right\}, \ i=1,2.
\end{equation}
The distance between these horospheres is constant:
\begin{equation} \label{eq:equidistance}
\dist\left(H^\gamma_{\tau_1}, H^\gamma_{\tau_2}\right)
=\dist\left(H^\gamma_{\tau_2}, H^\gamma_{\tau_1}\right)
=\left|\tau_2-\tau_1\right|.
\end{equation}
In particular, the point-to-horosphere distance is
\begin{equation}
    \dist\left(x, H^\gamma_\tau\right)=\left|B^\gamma(x) - \tau\right|, \ \forall x \in \calX.
\end{equation}
\end{theorem}
\begin{proof}
A Hadamard space does not need to be a manifold. It extends manifolds with nonpositive curvature to the broader setting of metric spaces. A brief review of metric geometry and the complete proof are given in \cref{app:subsec:metric-geometry,app:prf:thm:distance-horospheres}.
\end{proof}

\begin{corollary}[Point-to-horosphere distance]\label{cor:poincare-lorentz}
In a hyperbolic space $\hyperspace{n} \in \{ \pball{n}, \lorentz{n} \}$ with nonpositive curvature $K \leq 0$, the point-to-horosphere distance is
\begin{equation}
    \dist\left(x, H^v_\tau\right)=\left|B^v(x)-\tau\right|,
\end{equation}
where $H^v_\tau=\left\{ x \mid B^v(x)=\tau \right\}$ denotes the horosphere with respect to the direction $v \in \unitsphere{n-1}$.
\end{corollary}

A Euclidean hyperplane $H_{a,b}$ can be parameterized as $H_{v,\alpha,b}=\left\{ x \in \bbR{n} : \alpha \inner{v}{x}+b=0 \right\}$ with $v \in \unitsphere{n-1}$, $\alpha>0$, and $b \in \bbRscalar$. Similarly, we parameterize a hyperbolic horosphere as
\begin{equation} \label{eq:horosphere_param}
    H_{v,\alpha,b} = \left\{ x \in \hyperspace{n} \mid -\alpha B^v(x) + b = 0 \right\},
\end{equation}
with $v \in \unitsphere{n-1}$, $\alpha>0$, and $b \in \bbRscalar$. With this parameterization, we extend \cref{eq:euc-point-to-hyperplane} to hyperbolic spaces:
\begin{equation} \label{eq:horosphere-mlr}
    u_k(x) =  \sign_k \alpha_k \dist\left(x, H_{v_k,\alpha_k,b_k}\right)
\end{equation}
where $\sign_k = \sign\left(-\alpha_k B^{v_k}(x)+b_k\right)$, and $\{ v_k \in \unitsphere{n-1}, \alpha_k>0, b_k \in \bbRscalar \}$ are parameters for class $k$. By \cref{cor:poincare-lorentz}, the point-to-horosphere distance is
\begin{equation}\label{eq:point-to-horosphere-dist-param}
    \dist\left(x, H_{v,\alpha,b}\right) = \frac{\left|- \alpha B^v(x) + b\right|}{\alpha}.
\end{equation}
By \cref{eq:point-to-horosphere-dist-param}, \cref{eq:horosphere-mlr} equals the exact BMLR logit in \cref{eq:b-logits}.

\begin{remark}[Generality]
Since $B^v(x)=-\inner{v}{x}$ in Euclidean geometry, \cref{eq:horosphere_param,eq:point-to-horosphere-dist-param,eq:horosphere-mlr} naturally generalize to their Euclidean counterparts. We also acknowledge \citet[Eq. (2) and Prop. 3.1]{fan2023horospherical}, who used horospheres and point-to-horosphere distances to construct a hyperbolic SVM. However, they considered only the unit Poincaré ball with the specific curvature $K=-1$, which is a special case of our \cref{eq:horosphere_param,eq:point-to-horosphere-dist-param}.
\end{remark}

\subsection{Comparison}
Based on the point-to-hyperplane reformulation in \cref{eq:euc-point-to-hyperplane}, recent work extended MLR to the Poincaré \citep{ganea2018hyperbolic,shimizu2021hyperbolic,nguyen2025neural} and Lorentz \citep{bdeir2024fully} models. \citet[Sec. 3.1]{ganea2018hyperbolic} introduced the first Poincaré MLR by replacing the Euclidean point-to-hyperplane distance with its hyperbolic counterpart, where the hyperplane is defined by geodesics and the resulting distance is the real point-to-hyperplane distance, obtained as an infimum over the hyperplane. However, the formulation is not batch efficient (see \cref{app:sec:mlr-comparison}). It also requires per-class parameters $a_k \in T_{p_k} \pball{n}$ and $p_k \in \pball{n}$, which leads to over-parameterization. \citet[Sec. 3.1]{shimizu2021hyperbolic} alleviated such issues via re-parameterization. \citet[Sec. 4.3]{bdeir2024fully} further developed a Lorentz MLR, but its hyperplanes are defined by the ambient Minkowski space, which is tailored to the Lorentz model and does not fully respect the intrinsic hyperbolic geometry. Moreover, \citet[Cor. 4.3]{nguyen2025neural} proposed a Poincaré MLR based on the Busemann function. We refer to it as Pseudo-Busemann MLR, as the induced point-to-hyperplane distance is pseudo, coinciding with the real point-to-hyperplane distance only in Euclidean geometry. It also suffers from over-parameterization and is not batch efficient.

As summarized in \cref{tab:logit-comparison}\footnote{Relative to \citep[Def. 4.2, Cor. 4.3, and App. B.1.2]{nguyen2025neural}, the Pseudo-Busemann MLR written here includes an additional sign $-$; this is intentional and matches their official implementation.}, BMLR unifies advantages that prior hyperbolic MLRs offer only partially. In particular, BMLR respects the real point-to-horosphere distance, uses compact parameters without an additional manifold-valued point, attains the lowest FLOPs on $\lorentz{n}$ and a competitive cost on $\pball{n}$, and supports batch-efficient computation. On $\lorentz{n}$, its FLOPs are even close to those of the Euclidean MLR.

\section{Busemann fully connected layer}
We first revisit the Euclidean \glsfull{FC} layer from a geometric perspective, then propose the hyperbolic \emph{Busemann FC (BFC)} layer with its manifestations in the Poincaré and Lorentz models.

\subsection{Formulation}
\textbf{Euclidean FC layers.} An FC affine transformation $\mathcal{F}: \bbR{n} \ni x \mapsto y=Ax+b \in \bbR{m}$ can be expressed element-wise as $y_k = \inner{a_k}{x} + b_k$ with $a_k \in \bbR{n}$ and $b_k \in \bbRscalar$. As shown by \citet[Sec. 3.2]{shimizu2021hyperbolic} and \citet[Sec. 3.1]{chen2025building}, the LHS $y_k$ can be interpreted as the signed distance from $y$ to the hyperplane passing through the origin and orthogonal to the $k$-th axis of the output space. Hence, the FC layer can be expressed as
\begin{equation} \label{eq:euc-fc-p2h}
    \bar{\dist}\left(y, H_{e_k, 0}\right) 
    = \inner{a_k}{x} + b_k, \ \forall 1 \leq k \leq m,
\end{equation}
where $\bar{\dist}\left(y, H_{e_k, 0}\right) = \sign\left(\inner{e_k}{y}\right) \dist\left(y, H_{e_k, 0}\right)$ is the signed point-to-hyperplane distance, and $H_{e_k, 0}=\{y \in \bbR{m} \mid \inner{e_k}{y} = 0\}$ is a hyperplane with $e_k \in \bbR{m}$ as the vector whose $k$-th element is 1 and all others are 0.

\textbf{Lifting to hyperbolic space.}
To extend \cref{eq:euc-fc-p2h} into hyperbolic space, the RHS can be readily replaced by \cref{eq:b-logits}, as it generalizes the term $\inner{a_k}{x} + b_k$. For the LHS, a natural idea is to use the signed point-to-horosphere distance. However, as detailed in \cref{app:subsec:busemann-fc-p2h-distance}, this may fail to admit a solution for $y$. We therefore follow the point-to-hyperplane distance in \citep[Thm. 5]{ganea2018hyperbolic} for the Poincaré model and the one in \citep[Eq. (44)]{bdeir2024fully} for the Lorentz model. Given $x \in \hyperspace{n}$, the hyperbolic BFC layer $\calF: \hyperspace{n} \ni x \mapsto y \in \hyperspace{m}$ is given by solving $y$ via the following $m$ equations:
\begin{equation} \label{eq:hyp-fc-p2h}
    \bar{\dist}\left(y, H_{e_k, e}\right) = u_k(x), \ \forall 1 \leq k \leq m,
\end{equation}
where $u_k(x)= -\alpha_k B^{v_k}(x) + b_k$ with $\{ \alpha_k > 0, v_k \in \unitsphere{n-1}, b_k \in \bbRscalar \}$ as parameters. Here, $\bar{\dist} \left( y, H_{e_k,e} \right)$ is the hyperbolic signed distance from $y$ to the hyperplane passing through the origin $e \in \hyperspace{n}$. Next, we show that the above implicit definition has an explicit solution for the output $y$.

\begin{theorem}\label{thm:bfc-poincare}
Given an input $x \in \pball{n}$, the Poincaré BFC layer $\mathcal{F}: \pball{n} \to \pball{m}$ is given by
\begin{equation*}
    y = \frac{\omega}{1 + \sqrt{1 - K \norm{\omega}^{2}}}, \
    \omega =\left[\frac{\sinh \left(\sqrt{-K} u_k(x)\right)}{\sqrt{-K}}\right]_{k=1}^m,
\end{equation*}
where $u_k(x)= -\alpha_k B^{v_k}(x) + b_k$ with $\{ \alpha_k > 0, v_k \in \unitsphere{n-1}, b_k \in \bbRscalar \}$ as parameters for $\{k=1,\ldots,m\}$.
\end{theorem}
\begin{proof}
    The proof, together with $H_{e_k, e}$ and $\bar{\dist}\left(y, H_{e_k, e}\right)$, is provided in \cref{app:prf:thm:bfc-poincare}, which is inspired by \citet[App. D.3]{shimizu2021hyperbolic}.
\end{proof}

\begin{theorem}\label{thm:bfc-lorentz}
Given an input $x \in \lorentz{n}$, the Lorentz BFC layer $\calF: \lorentz{n} \to \lorentz{m}$ is given by 
\begin{equation*}
    y = \begin{bmatrix}
    y_t\\
    y_s
    \end{bmatrix}
    =
    \begin{bmatrix}
    \sqrt{\frac{1}{-K}+\norm{y_s}^{2}} \\[6pt]
    \frac{1}{\sqrt{-K}}\sinh\left(\sqrt{-K} u(x)\right)
    \end{bmatrix}
\end{equation*}
where $u(x)=\left(u_1(x),\ldots,u_m(x)\right)^{\top}$ with $u_k(x)= -\alpha_k B^{v_k}(x) + b_k$. Here, $\{ \alpha_k > 0, v_k \in \unitsphere{n-1}, b_k \in \bbRscalar \}$ are parameters for $\{k=1,\ldots,m\}$.
\end{theorem}
\begin{proof}
    The proof, along with the hyperplane and point-to-hyperplane distance, is provided in \cref{app:prf:thm:bfc-lorentz}.
\end{proof}

Analogously to \cref{thm:limits-bmlr}, our BFC layers converge to their Euclidean counterparts as $K \to 0^{-}$.
\begin{theorem}[Limits of BFC layers]\label{thm:limits-bfc}
As $K \to 0^{-}$, the hyperbolic BFC layer $\hyperspace{n} \ni x \mapsto y \in \hyperspace{m}$ reduces to a Euclidean FC layer:
\begin{align}
\text{(Poincaré)} \quad & y_k \xrightarrow{K\to 0^{-}} \alpha_k \inner{v_k}{x} + \frac{1}{2} b_k, \\
\text{(Lorentz)} \quad & (y_s)_k \xrightarrow{K\to 0^{-}} \alpha_k \inner{v_k}{x_s} + b_k.
\end{align}
\end{theorem}
\begin{proof}
    The proof is provided in \cref{app:prf:thm:limits-bfc}.
\end{proof}

\begin{table*}[t]
    \centering
    \caption{Comparison of hyperbolic FC layers. For simplicity, BFC layers do not involve the gyroaddition and assume $\phi$ is the identity map, which is in line with the Möbius and Lorentz FC layers.}
    \label{tab:fc-comparison}
    \resizebox{\linewidth}{!}{
    \begin{tabular}{ccccccc}
        \toprule
        Method & $\calF: \hyperspace{n} \ni x \mapsto y \in \hyperspace{m}$ & Space & Methodology & Parameters & \#Params & FLOPs \\
        \midrule
        \makecell{M\"obius  \\ \citep[Eq. 27]{ganea2018hyperbolic}} &
        $\displaystyle \frac{1}{\sqrt{-K}} \tanh\left( \frac{\norm{W x}}{\norm{x}} \tanh^{-1}\left(\sqrt{-K} \norm{x}\right)\right) \frac{W x}{\norm{W x}}$ &
        $\pball{n}$ & Tangent & $W \in \bbR{m \times n}$ & $mn$ & \makecell{$2nm+2n$ \\ $+2m+24$} \\
        \midrule
        \makecell{Poincar\'e FC \\ \citep[Eq. (7)]{shimizu2021hyperbolic}} &
        \makecell{$\displaystyle y = \frac{\omega}{1 + \sqrt{1 - K \norm{\omega}^{2}}}, \ \omega_k = \frac{\sinh\left(\sqrt{-K} u_k(x)\right)}{\sqrt{-K}}$, \\ with $u_k(x)$ in \cref{tab:logit-comparison}} &
        $\pball{n}$ & \makecell{Poincar\'e \\ geometry} & \makecell{$\alpha_k>0,\ v_k \in \unitsphere{n-1}$,\\ $b_k \in \bbRscalar$, for $k=1,\ldots,m$} & $m(n+2)$ & $4nm+71m+4$\\
        \midrule
        \makecell{Lorentz FC \\ \citep[Eq. (3)]{chen2022fully}} &
        \makecell{$\displaystyle y = \begin{bmatrix} \sqrt{\norm{\psi(Wx, v)}^{2} - 1/K} \\ \psi(Wx, v) \end{bmatrix}$, \\ $\displaystyle \psi(Wx, v) = \lambda \sigma\left( v^{\top} x + b' \right) \frac{ W \phi(x) + b }{ \norm{ W \phi(x) + b } }$ \\ with $\phi$ and $\sigma$ as the activation and sigmoid} &
        $\lorentz{n}$ & \makecell{ Ambient \\ Minkowski} & \makecell{$W \in \bbR{m\times (n+1)}$,\\ $v \in \bbR{n+1}$, $b \in \bbR{m}$,\\ $b' \in \bbRscalar$, $\lambda>0$} & \makecell{$m(n+1)+m$ \\ $+(n+1)+2$} & \makecell{$2nm+8m$ \\ $+2n+10$}\\
        \midrule
        \rowcolor{HilightColor} BFC &
        \Gape[0pt][2pt]{\makecell{$\displaystyle y = \frac{\omega}{1 + \sqrt{1 - K \norm{\omega}^{2}}}$, \ $\displaystyle \omega = \frac{\sinh\left(\sqrt{-K} u(x)\right)}{\sqrt{-K}}$; \\ $\displaystyle y_s = \frac{1}{\sqrt{-K}} \sinh\left(\sqrt{-K} u(x)\right)$, \ $\displaystyle y_t = \sqrt{\frac{1}{-K} + \norm{y_s}^{2}}$, \\ with $\displaystyle u_k(x) = \phi(-\alpha_k B^{v_k}(x) + b_k)$}} &
        \Gape[0pt][2pt]{\makecell{$\pball{n}$ \\ $\lorentz{n}$}} & Busemann & \makecell{$\alpha_k>0,\ v_k \in \unitsphere{n-1}$,\\ $b_k \in \bbRscalar$, for $k=1,\ldots,m$} & $m(n+2)$ & \makecell{$6nm+29m+4$ \\ $2nm+30m+2$}\\
        \bottomrule
    \end{tabular}
    }
    \vspace{-3mm}
\end{table*}

\subsection{Generalization}\label{subsec:busemann-fc-generalization}
When incorporating an activation function $\phi: \bbRscalar \to \bbRscalar$, we can jointly express the Euclidean FC and activation layers, which yields \cref{eq:euc-fc-p2h} with $u_k(x)=\phi\left(\inner{a_k}{x} + b_k\right)$. Accordingly, we extend the hyperbolic BFC by inserting the activation into \cref{eq:hyp-fc-p2h}:
\begin{equation}
    \bar{\dist}\left(y, H_{e_k, e}\right) = \phi\left(u_k(x)\right), \ \forall 1 \leq k \leq m.
\end{equation}
This is reflected in \cref{thm:bfc-lorentz,thm:limits-bfc} by replacing every $u_k(x)$ with $\phi\left(-\alpha_k B^{v_k}(x)+b_k\right)$. Moreover, inspired by the Poincaré Möbius transformation \citep[Sec. 3.2]{ganea2018hyperbolic}, a BFC transformation could be further followed by a gyroaddition $\Hoplus$: $\hyperspace{n} \ni x \mapsto \calF(x) \Hoplus b \in \hyperspace{m}$ with $b \in \hyperspace{m}$ as a gyro bias. For example, the Lorentz BFC layer is generalized as
\begin{equation*}
    \lorentz{n} \ni x \mapsto y =
    \begin{bmatrix}
    \sqrt{\frac{1}{-K}+\norm{y_s}^{2}} \\[6pt]
    \frac{1}{\sqrt{-K}}\sinh\left(\sqrt{-K} u(x)\right)
    \end{bmatrix} \Loplus b \in \lorentz{m},
\end{equation*}
where $u_k(x)=\phi\left(-\alpha_k B^{v_k}(x)+b_k\right)$ with parameters $\{ \alpha_k > 0, v_k \in \unitsphere{n-1}, b_k \in \bbRscalar \}^m_{k=1}$ and $b \in \lorentz{m}$.

\subsection{Comparison}
\cref{tab:fc-comparison} compares BFC with prior hyperbolic FC layers. BFC faithfully respects hyperbolic geometry, whereas the M\"obius and Lorentz FC layers apply Euclidean transformations in the tangent or ambient Minkowski space, which can distort intrinsic geometry. BFC also offers flexibility across models, while Poincar\'e FC and Lorentz FC are tailored to their respective models. In addition, BFC uses a comparable parameterization and maintains $\calO(nm)$ FLOPs. On $\lorentz{n}$, its FLOPs are $\calO(2mn)$, matching the fastest layers.

\section{Experiments}

We first compare BMLRs with prior hyperbolic MLRs on three architectures: ResNet-18 (image classification), CNN (genome sequences), and HGCN (node classification). We then compare BFC with prior hyperbolic FC layers on link prediction. All experiments use both the Poincar\'e and Lorentz models.

\begin{table*}[t]
  \centering
  \caption{Top-1 image classification accuracy (\%) of MLR methods on the ResNet-18 backbone. Best results within each hyperbolic model are in \textbf{bold}. Fit time (s/epoch) highlights: fastest hyperbolic MLR in \greenbf{green}, slowest MLR in \redbf{red}. \#Params denotes the number of parameters in the MLR head, with the largest marked in \redbf{red}.}
  \label{tab:exp-image-classification}
  \resizebox{\linewidth}{!}{
    \begin{tabular}{c|c|ccc|ccc|ccc|ccc}
    \toprule
    \multirow{2}[4]{*}{Space} & \multirow{2}[4]{*}{Method} & \multicolumn{3}{c|}{\makecell{CIFAR-10 \\ (Num. classes: 10)}} & \multicolumn{3}{c|}{\makecell{CIFAR-100 \\ (Num. classes: 100)}} & \multicolumn{3}{c|}{\makecell{Tiny-ImageNet \\ (Num. classes: 200)}} & \multicolumn{3}{c}{\makecell{ImageNet-1k \\ (Num. classes: 1000)}} \\
\cmidrule{3-14}          &       & Acc   & Fit Time &  \#Params & Acc   & Fit Time &  \#Params & Acc   & Fit Time &  \#Params & Acc   & Fit Time &  \#Params \\
    \midrule
    $\bbR{n}$ & MLR   & 95.14 ± 0.12 & 10.66 &  5.13K & 77.72 ± 0.15 & 10.60  &  51.30K & 65.19 ± 0.12 & 69.17  &  102.60K & 71.87 & 2263.12  & 513K \\
    \midrule
    \multirow{3}[2]{*}{$\pball{n}$} & PMLR & 95.04 ± 0.13 & 11.94 &  5.14K & 77.19 ± 0.50 & 12.11  &  51.40K & 64.93 ± 0.38 & 71.90  &  102.80K & 71.77 & 2300.11  & 514K \\
          & PBMLR-P & 95.23 ± 0.08 & \redbf{21.92} & \redbf{ 10.24K} & 77.78 ± 0.15 & \redbf{76.84 } & \redbf{ 102.40K} & 65.43 ± 0.27 & \redbf{336.58 } & \redbf{ 204.80K} & 71.46 & \redbf{3907.12 } & \redbf{1024K} \\
          & \cellcolor{HilightColor} BMLR-P & \cellcolor{HilightColor}\textbf{95.32 ± 0.14} & \cellcolor{HilightColor}12.01 & \cellcolor{HilightColor} 5.14K & \cellcolor{HilightColor}\textbf{78.10 ± 0.35} & \cellcolor{HilightColor}12.13  & \cellcolor{HilightColor} 51.40K & \cellcolor{HilightColor}\textbf{66.16 ± 0.19} & \cellcolor{HilightColor}71.98  & \cellcolor{HilightColor} 102.80K & \cellcolor{HilightColor}\textbf{73.36} & \cellcolor{HilightColor}2300.77  & \cellcolor{HilightColor}514K \\
    \midrule
    \multirow{2}[2]{*}{$\lorentz{n}$} & LMLR  & 94.98 ± 0.12 & 11.55 &  5.13K & 78.03 ± 0.21 & 11.72  &  51.30K & 65.63 ± 0.10 & 69.27  &  102.60K & 72.46 & 2277.17  & 513K \\
          & \cellcolor{HilightColor} BMLR-L & \cellcolor{HilightColor}\textbf{95.25 ± 0.02} & \cellcolor{HilightColor}\greenbf{11.08} & \cellcolor{HilightColor} 5.14K & \cellcolor{HilightColor}\textbf{78.07 ± 0.26} & \cellcolor{HilightColor}\greenbf{11.22}  & \cellcolor{HilightColor} 51.40K & \cellcolor{HilightColor}\textbf{65.99 ± 0.14} & \cellcolor{HilightColor}\greenbf{69.19}  & \cellcolor{HilightColor} 102.80K & \cellcolor{HilightColor}\textbf{73.24} & \cellcolor{HilightColor}\greenbf{2276.53}  & \cellcolor{HilightColor}514K \\
    \bottomrule
    \end{tabular}%
  }
\end{table*}%

\begin{figure}[t]
\centering
\includegraphics[width=\linewidth,trim=0 0cm 0 0]{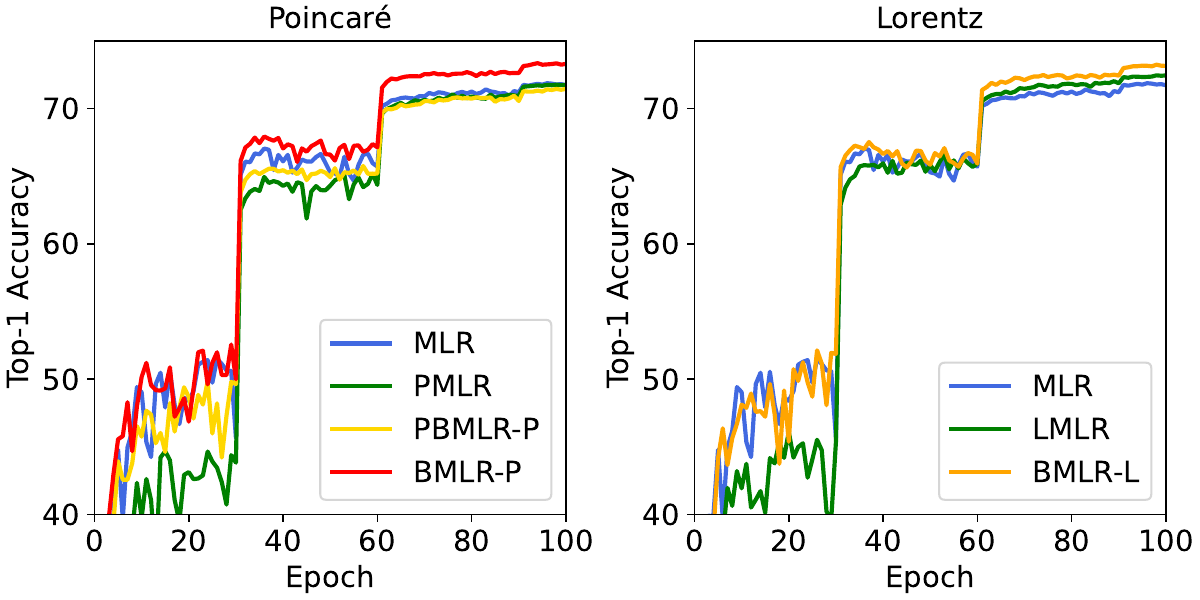}
\caption{Validation accuracy curves on ImageNet-1k.}
\label{fig:acc_curves_imagenet}
\vspace{-5mm}
\end{figure}

\begin{table*}[t]
  \centering
  \caption{Genomic MCC of MLR methods under the CNN backbone. Best results within each hyperbolic model are in \textbf{bold}.}
    \label{tab:exp-genome-sequence-learning}%
    \resizebox{\linewidth}{!}{
    \begin{tabular}{ccc|c|ccc|cc}
    \toprule
    \multirow{2}[3]{*}{Benchmark} & \multirow{2}[3]{*}{Task} & \multirow{2}[3]{*}{Dataset} & Num. & \multicolumn{3}{c|}{$\pball{n}$} & \multicolumn{2}{c}{$\lorentz{n}$} \\
\cmidrule{5-9}          &       &       & classes & PMLR  & PBMLR-P & \cellcolor{HilightColor}BMLR & LMLR  & \cellcolor{HilightColor}BMLR-L \\
\midrule
    \multirow{7}[5]{*}{TEB} & \multirow{3}[1]{*}{Retrotransposons} & LTR Copia & 2     & 75.34 ± 1.02 & 74.37 ± 1.48 & \cellcolor{HilightColor}\textbf{76.73 ± 1.08} & 73.01 ± 1.07 & \cellcolor{HilightColor}\textbf{75.86 ± 1.52} \\
          &       & LINEs & 2     & 85.54 ± 0.61 & 85.92 ± 0.65 & \cellcolor{HilightColor}\textbf{86.05 ± 1.08} & 83.14 ± 0.80 & \cellcolor{HilightColor}\textbf{86.72 ± 0.58} \\
          &       & SINEs & 2     & 95.30 ± 0.85 & 95.34 ± 1.58 & \cellcolor{HilightColor}\textbf{95.99 ± 0.74} & \textbf{96.70 ± 0.87} & \cellcolor{HilightColor}96.29 ± 0.59 \\
\cmidrule{2-9}          & \multirow{2}[2]{*}{DNA transposons} & CMC-EnSpm & 2     & 83.39 ± 0.56 & 83.62 ± 1.00 & \cellcolor{HilightColor}\textbf{84.03 ± 0.71} & 81.78 ± 1.05 & \cellcolor{HilightColor}\textbf{84.15 ± 1.00} \\
          &       & hAT-Ac & 2     & 89.38 ± 0.90 & \textbf{89.86 ± 0.54} & \cellcolor{HilightColor}89.62 ± 0.74 & 88.94 ± 0.69 & \cellcolor{HilightColor}\textbf{90.70 ± 0.51} \\
\cmidrule{2-9}          & \multirow{2}[2]{*}{Pseudogenes} & processed & 2     & 72.45 ± 1.49 & 71.99 ± 2.04 & \cellcolor{HilightColor}\textbf{73.09 ± 1.66} & \textbf{73.71 ± 1.76} & \cellcolor{HilightColor}73.32 ± 1.65 \\
          &       & unprocessed & 2     & 75.37 ± 2.27 & 71.99 ± 1.47 & \cellcolor{HilightColor}\textbf{75.71 ± 1.89} & 74.54 ± 1.98 & \cellcolor{HilightColor}\textbf{76.15 ± 1.61} \\
    \midrule
    \multirow{9}[8]{*}{GUE} & \multirow{3}[2]{*}{Core Promoter Detection} & tata  & 2     & \textbf{80.95 ± 1.47} & 79.32 ± 2.44 & \cellcolor{HilightColor}80.29 ± 1.63 & 80.90 ± 1.15 & \cellcolor{HilightColor}\textbf{81.76 ± 1.16} \\
          &       & notata & 2     & 70.02 ± 0.52 & \textbf{70.60 ± 0.75} & \cellcolor{HilightColor}70.48 ± 0.35 & \textbf{71.26 ± 0.56} & \cellcolor{HilightColor}70.43 ± 0.39 \\
          &       & all   & 2     & 67.64 ± 0.77 & 68.02 ± 0.63 & \cellcolor{HilightColor}\textbf{68.50 ± 0.61} & 67.63 ± 0.56 & \cellcolor{HilightColor}\textbf{68.36 ± 1.07} \\
\cmidrule{2-9}          & \multirow{3}[2]{*}{Promoter Detection} & tata  & 2     & 80.30 ± 1.59 & 80.27 ± 2.71 & \cellcolor{HilightColor}\textbf{82.83 ± 1.69} & \textbf{83.27 ± 1.95} & \cellcolor{HilightColor}82.55 ± 1.54 \\
          &       & notata & 2     & 92.63 ± 0.36 & \textbf{93.05 ± 0.32} & \cellcolor{HilightColor}92.75 ± 0.51 & 91.74 ± 0.57 & \cellcolor{HilightColor}\textbf{92.60 ± 0.49} \\
          &       & all   & 2     & 90.53 ± 0.50 & \textbf{90.79 ± 0.77} & \cellcolor{HilightColor}90.20 ± 0.65 & 89.34 ± 0.40 & \cellcolor{HilightColor}\textbf{89.82 ± 0.45} \\
\cmidrule{2-9}          & Covid Variant Classification & Covid & 9     & \textbf{74.09 ± 0.25} & 70.84 ± 0.80 & \cellcolor{HilightColor}73.40 ± 0.30 & 64.07 ± 0.51 & \cellcolor{HilightColor}\textbf{72.45 ± 0.21} \\
\cmidrule{2-9}          & \multirow{2}[2]{*}{Species Classification} & Virus & 20    & 67.24 ± 2.10 & 59.17 ± 3.32 & \cellcolor{HilightColor}\textbf{77.12 ± 1.23} & 71.34 ± 2.05 & \cellcolor{HilightColor}\textbf{77.21 ± 1.04} \\
          &       & Fungi & 25    & 15.06 ± 1.32 & 18.75 ± 1.77 & \cellcolor{HilightColor}\textbf{30.01 ± 0.76} & 15.07 ± 1.88 & \cellcolor{HilightColor}\textbf{30.14 ± 2.48} \\
    \bottomrule
    \end{tabular}
    }
    \vspace{-3mm}
\end{table*}%

\subsection{Image classification}
\label{subsec:exp-image-classification}

\textbf{Setup.}
Following \citet[Sec. 5.1]{bdeir2024fully} and \citet[Sec. 4]{guo2022clipped}, we use a hybrid architecture with a ResNet-18 \citep{he2016deep} backbone and an MLR head. We compare Euclidean MLR with hyperbolic variants in both models. In Poincar\'e, we evaluate Poincaré MLR (PMLR) re-parameterized by \citet{shimizu2021hyperbolic}, Pseudo Busemann MLR (PBMLR-P) \citep{nguyen2025neural}, and our BMLR-P. In Lorentz, we evaluate Lorentz MLR (LMLR) \citep{bdeir2024fully} and our BMLR-L. For hyperbolic MLRs, we map the ResNet-18 features to the target hyperbolic space before classification. We evaluate on CIFAR-10 \citep{krizhevsky2009learning}, CIFAR-100 \citep{krizhevsky2009learning}, Tiny-ImageNet \citep{le2015tiny}, and ImageNet-1k \citep{deng2009imagenet}. On the first three datasets, we conduct five-fold experiments. More details appear in \cref{app:subsec:image-classification-details}.

\textbf{Results.}
\cref{tab:exp-image-classification} reports top-1 validation accuracy, fit time per epoch, and classifier-head parameters. \cref{fig:acc_curves_imagenet} presents the ImageNet-1k accuracy curves. Overall, BMLR-P and BMLR-L consistently outperform prior hyperbolic MLRs with comparable parameters. Within each hyperbolic model, the accuracy margin over prior hyperbolic MLRs increases with the number of classes, from CIFAR-10 to CIFAR-100 and Tiny-ImageNet, with the largest gains on ImageNet-1k. This demonstrates the advantage of BMLR as task complexity increases. Besides, PBMLR-P uses approximately double the head parameters and is markedly slower due to complex batch-inefficient computation (see \cref{app:sec:mlr-comparison}), whereas BMLR-L achieves the fastest fit time among all hyperbolic MLRs.

\subsection{Genome sequence learning}

\begin{table}[t]
  \centering
  \caption{Fit time (s/epoch) on genome sequence learning. The fastest times are in \greenbf{green} and the slowest ones are in \redbf{red}.}
  \label{tab:exp-genome-sequence-learning-efficiency}%
  \resizebox{0.9\linewidth}{!}{
    \begin{tabular}{c|ccc|cc}
    \toprule
    \multirow{2}[4]{*}{Dataset} & \multicolumn{3}{c|}{$\pball{n}$} & \multicolumn{2}{c}{$\lorentz{n}$} \\
\cmidrule{2-6}          & PMLR  & PBMLR-P & \cellcolor{HilightColor}BMLR & LMLR  & \cellcolor{HilightColor}BMLR-L \\
    \midrule
    LTR Copia & 3.96  & \redbf{5.11} & \cellcolor{HilightColor}4.06  & 3.92  & \cellcolor{HilightColor}\greenbf{3.77} \\
    LINEs & 5.80  & \redbf{6.90} & \cellcolor{HilightColor}5.73  & 5.50  & \cellcolor{HilightColor}\greenbf{5.32} \\
    SINEs & 1.36  & \redbf{1.80} & \cellcolor{HilightColor}1.39  & 1.38  & \cellcolor{HilightColor}\greenbf{1.28} \\
    \midrule
    CMC-EnSpm & 3.11  & \redbf{4.38} & \cellcolor{HilightColor}3.04  & 2.89  & \cellcolor{HilightColor}\greenbf{2.82} \\
    hAT-Ac & 4.37  & \redbf{5.37} & \cellcolor{HilightColor}4.38  & 4.13  & \cellcolor{HilightColor}\greenbf{3.96} \\
    \midrule
    processed & 5.02  & \redbf{5.94} & \cellcolor{HilightColor}4.89  & 4.68  & \cellcolor{HilightColor}\greenbf{4.58} \\
    unprocessed & 3.30  & \redbf{3.90} & \cellcolor{HilightColor}3.29  & 3.05  & \cellcolor{HilightColor}\greenbf{2.95} \\
    \midrule
    CPD-tata & 0.72  & \redbf{1.35} & \cellcolor{HilightColor}0.70  & 0.67  & \cellcolor{HilightColor}\greenbf{0.62} \\
    CPD-notata & 6.13  & \redbf{12.40} & \cellcolor{HilightColor}6.00  & 5.73  & \cellcolor{HilightColor}\greenbf{5.71} \\
    CPD-all & 6.85  & \redbf{14.35} & \cellcolor{HilightColor}6.59  & 6.35  & \cellcolor{HilightColor}\greenbf{6.33} \\
    \midrule
    PD-tata & 0.98  & \redbf{1.20} & \cellcolor{HilightColor}0.97  & 1.04  & \cellcolor{HilightColor}\greenbf{0.94} \\
    PD-notata & 8.33  & \redbf{10.53} & \cellcolor{HilightColor}8.29  & 8.11  & \cellcolor{HilightColor}\greenbf{7.84} \\
    PD-all & 9.40  & \redbf{12.07} & \cellcolor{HilightColor}9.19  & 9.06  & \cellcolor{HilightColor}\greenbf{8.83} \\
    \midrule
    Covid & 28.96  & \redbf{45.52} & \cellcolor{HilightColor}27.97  & 27.58  & \cellcolor{HilightColor}\greenbf{26.67} \\
    \midrule
    Virus & 25.28  & \redbf{29.57} & \cellcolor{HilightColor}25.67  & 25.12  & \cellcolor{HilightColor}\greenbf{24.85} \\
    Fungi & 6.41  & \redbf{8.96} & \cellcolor{HilightColor}6.41  & \greenbf{6.25}  & \cellcolor{HilightColor}\greenbf{6.25} \\
    \bottomrule
    \end{tabular}%
  }
  \vspace{-3mm}
\end{table}%

\textbf{Setup.}
Recently, \citet{khan2025hyperbolic} demonstrated the effectiveness of hyperbolic embedding in genome sequence learning. Following their settings, we adopt a CNN backbone, which consists of three convolutional blocks and an MLR head. Similar to \cref{subsec:exp-image-classification}, we compare our BMLR against previous hyperbolic MLR heads by replacing the final Euclidean MLR with a hyperbolic MLR. We validate on two benchmarks: Transposable Element Benchmark (TEB) \citep{khan2025hyperbolic} and Genome Understanding Evaluation (GUE) \citep{zhou2024dnabert}, covering a total of 16 datasets. More details are provided in \cref{app:subsec:genome-sequence-learning-details}.

\textbf{Results.}
\cref{tab:exp-genome-sequence-learning} summarizes 5-fold average \glsfull{MCC} across TEB and GUE. Compared with other hyperbolic MLRs, our BMLR-P and BMLR-L achieve higher MCC in most tasks. Similar to \cref{subsec:exp-image-classification}, the gains are more pronounced on complex datasets with more classes, \eg, Virus (20 classes) and Fungi (25 classes), demonstrating the effectiveness of our approach. \cref{tab:exp-genome-sequence-learning-efficiency} reports fit time per epoch, where PBMLR-P is consistently the slowest due to batch inefficiency, and BMLR-L is the fastest.

\begin{table*}[!t]
  \centering
  \caption{Comparison of hyperbolic FC layers on link prediction. Best results within each hyperbolic model are in \textbf{bold}.}
  \label{tab:exp-lp-hnn}%
  \resizebox{0.85\linewidth}{!}{
    \begin{tabular}{c|cc|cccc}
    \toprule
    \multirow{2}[2]{*}{Space} & \multirow{2}[2]{*}{Method} & \multirow{2}[2]{*}{Methodology} & Disease & Airport & PubMed & Cora \\
          &       &       & $\delta=0$     & $\delta=1$     & $\delta=3.5$ & $\delta=11$ \\
    \midrule
    \multirow{3}[2]{*}{$\pball{n}$} & Möbius & Tangent & 76.35 ± 1.83 & 93.31 ± 0.41 & \textbf{94.93 ± 0.06} & 90.80 ± 0.56 \\
          & Poincaré FC & Poincaré geometry & 79.45 ± 1.01 & 94.31 ± 0.16 & 94.24 ± 0.25 & 88.21 ± 0.72 \\
          & \cellcolor{HilightColor} BFC-P & \cellcolor{HilightColor} Busemann & \cellcolor{HilightColor}\textbf{80.45 ± 0.93} & \cellcolor{HilightColor}\textbf{94.88 ± 0.39} & \cellcolor{HilightColor}94.85 ± 0.07 & \cellcolor{HilightColor}\textbf{91.94 ± 0.32} \\
    \midrule
    \multirow{3}[2]{*}{$\lorentz{n}$} & LTFC  & Tangent & 71.32 ± 5.36 & 92.68 ± 0.35 & 94.85 ± 0.17 & 89.37 ± 0.64 \\
          & Lorentz FC & Ambient Minkowski & 72.78 ± 2.04 & 92.99 ± 0.33 & 94.20 ± 0.10 & 92.06 ± 0.62 \\
          & \cellcolor{HilightColor} BFC-L & \cellcolor{HilightColor} Busemann & \cellcolor{HilightColor}\textbf{78.36 ± 0.51} & \cellcolor{HilightColor}\textbf{95.37 ± 0.17} & \cellcolor{HilightColor}\textbf{94.90 ± 0.04} & \cellcolor{HilightColor}\textbf{92.28 ± 0.12} \\
    \bottomrule
    \end{tabular}%
  }
  \vspace{-3mm}
\end{table*}%

\begin{table}[t]
  \centering
  \caption{Node classification F1 scores of hyperbolic MLRs on the HGCN backbone, where $\delta$ denotes the graph hyperbolicity (lower is more hyperbolic). The best results within each hyperbolic model are highlighted in \textbf{bold}.}
  \label{tab:exp-node-classification}%
  \resizebox{\linewidth}{!}{
    \begin{tabular}{c|c|cccc}
    \toprule
    \multirow{2}[3]{*}{Space} & \multirow{2}[3]{*}{Method} & Disease & Airport & PubMed & Cora \\
     &  & $\delta=0$     & $\delta=1$     & $\delta=3.5$ & $\delta=11$ \\
    \midrule
    \multirow{4}[2]{*}{$\pball{n}$} & HGCN  & 86.87 ± 2.58 & 85.34 ± 1.16 & 76.29 ± 0.98 & 76.56 ± 0.81 \\
          & HGCN-PMLR & 88.98 ± 1.96 & 84.78 ± 1.48 & 76.02 ± 1.09 & 77.47 ± 1.15 \\
          & HGCN-PBMLR-P & 89.05 ± 0.78 & 85.04 ± 0.97 & 75.89 ± 0.78 & 77.90 ± 1.00 \\
          & \cellcolor{HilightColor} HGCN-BMLR-P & \cellcolor{HilightColor}\textbf{92.45 ± 0.96} & \cellcolor{HilightColor}\textbf{86.02 ± 0.53} & \cellcolor{HilightColor}\textbf{77.36 ± 0.73} & \cellcolor{HilightColor}\textbf{78.48 ± 1.52} \\
    \midrule
    \multirow{3}[2]{*}{$\lorentz{n}$} & HGCN  & 87.83 ± 0.77 & 84.94 ± 1.40 & 76.49 ± 0.88 & 77.37 ± 1.72 \\
          & HGCN-LMLR & 89.72 ± 1.51 & 82.61 ± 1.01 & 75.44 ± 1.17 & 69.91 ± 3.61 \\
          & \cellcolor{HilightColor} HGCN-BMLR-L & \cellcolor{HilightColor}\textbf{90.80 ± 1.15} & \cellcolor{HilightColor}\textbf{85.27 ± 1.17} & \cellcolor{HilightColor}\textbf{77.30 ± 0.41} & \cellcolor{HilightColor}\textbf{77.65 ± 2.10} \\
    \bottomrule
    \end{tabular}%
  }
  \vspace{-3mm}
\end{table}%

\subsection{Node classification}

\textbf{Setup.}
Following \citet{nguyen2025neural}, we adopt the HGCN \citep{chami2019hyperbolic} backbone to evaluate our BMLR on graph datasets, including Disease \citep{anderson1991infectious}, Airport \citep{zhang2018link}, PubMed \citep{namata2012query}, and Cora \citep{sen2008collective}. The HGCN backbone consists of a hyperbolic Graph Convolutional Network (GCN) and an MLR as the final classification layer. Both the GCN and the MLR are built on the hyperbolic space. The vanilla HGCN uses a tangent MLR, which maps features into the tangent space via $\rielog_e$ and applies a Euclidean MLR. We replace this with different hyperbolic MLRs. More details are provided in \cref{app:subsec:node-classification-details}.

\textbf{Results.}
\cref{tab:exp-node-classification} reports average F1 scores. Our BMLRs consistently outperform prior hyperbolic MLRs within each hyperbolic model. As graphs become less hyperbolic, that is, for larger $\delta$, existing hyperbolic heads could underperform the vanilla tangent-based MLR, for example, PBMLR-P on PubMed, and LMLR on Airport, PubMed, and Cora. Especially on Cora, which has the largest $\delta$, LMLR lags the tangent baseline by a large margin (69.91 vs. 77.37). In contrast, BMLR remains the top performer across all $\delta$ values, indicating that Busemann-based decoding robustly strengthens HGCN over a broader range of graph hyperbolicity.

\subsection{Link prediction}

\textbf{Setup.}
We compare our BFC layers with prior hyperbolic FC layers, including the Möbius layer \citep{ganea2018hyperbolic} that operates via the tangent space, the Lorentz FC layer \citep{chen2022fully} that operates through the ambient Minkowski space, and the Poincar\'e FC layer \citep{shimizu2021hyperbolic}. Mimicking the Möbius layer, we further implement a Lorentz tangent FC layer, $\rielog_{\Lzero} (M\rielog_{\Lzero}(x))$, referred to as LTFC. Following \citet{chami2019hyperbolic}, we evaluate on four graph datasets for the link prediction task: Disease \citep{anderson1991infectious}, Airport \citep{zhang2018link}, PubMed \citep{namata2012query}, and Cora \citep{sen2008collective}. Following the HNN implementation \citep{ganea2018hyperbolic,chami2019hyperbolic}, all methods share the same backbone with two FC layers. For a fair comparison, all hyperbolic FC layers are followed by a gyroaddition biasing. For BFC, we set $\phi=\tanh$ on Airport and Cora, which yields better performance, while we use the identity map on the other two datasets. More details are provided in \cref{app:subsec:link-prediction-details}.

\textbf{Results.}
\cref{tab:exp-lp-hnn} reports 5-fold test AUC. Our BFC layers generally outperform prior hyperbolic FC layers. The gains are most pronounced on Disease, which is the most hyperbolic ($\delta=0$), where Busemann-based decoding is markedly more effective than tangent or ambient methods, indicating a better capture of intrinsic hyperbolic geometry. This observation aligns with geometric intuition, since tangent space or ambient space approximations inherently struggle to represent curved manifolds in highly non-Euclidean cases. \cref{app:subsec:albations-efficiency-bfc} summarizes fit time and parameter counts. LTFC is the slowest due to costly logarithmic and exponential maps, and LFC uses the largest number of parameters among Lorentz variants. In contrast, our BFC layers achieve training time and model size comparable to those of existing layers.

\section{Conclusion}
We introduce BMLR and BFC as intrinsic components for hyperbolic neural networks on the Poincaré and Lorentz models, both built from the Busemann function. BMLR provides compact parameters, a point-to-horosphere interpretation, batch-efficient computation, and a limit that recovers Euclidean MLR. BFC extends FC and activation layers with practical $O(nm)$ complexity. Experiments across image classification, genome sequence learning, node classification, and link prediction validate the effectiveness and efficiency of our approaches. BMLR shows increasingly strong performance as the number of classes increases, and replacing existing hyperbolic FC layers with BFC yields additional gains. These results indicate that Busemann geometry offers unified and effective mathematical tools for building hyperbolic neural networks.
\section*{Acknowledgements}
This work was supported by EU Horizon project ELLIOT (No. 101214398) and by the FIS project GUIDANCE (No. FIS2023-03251). We acknowledge CINECA for awarding high-performance computing resources under the ISCRA initiative, and the EuroHPC Joint Undertaking for granting access to Leonardo at CINECA, Italy.

{
    \small
    \bibliographystyle{ieeenat_fullname}
    \bibliography{ref,ref_cof,ref_others}
}

\clearpage
\appendix
\onecolumn
\crefalias{section}{appendix}
\begin{center}
\Large
\textbf{\thetitle}\\
\vspace{0.5em}Supplementary Material \\
\vspace{1.0em}
\end{center}

\startcontents[appendices]
\printcontents[appendices]{l}{1}{\section*{Appendix Contents}}
\newpage


\section{Notation} 
\label{app:sec:notation}
\cref{app:tab:notation} summarizes the notation used throughout the paper.

\begin{table}[t]
    \centering
    \caption{Summary of notation.}
    \label{app:tab:notation}
    \begin{tabular}{cc}
    \toprule
    Notation & Description \\
    \midrule
    $\calM$ & Riemannian manifold \\
    $T_x\calM$ & Tangent space at $x$ \\
    $\rieexp_x(v)$ & Exponential map at $x$ \\
    $\rielog_x(y)$ & Logarithmic map at $x$ \\
    $\pt{x}{y}(v)$ & Parallel transport of $v$ from $x$ to $y$ \\
    $g_x(u,v)$ & Riemannian metric at $x$ \\
    $\dist(x,y)$ & Geodesic distance \\
    $\gamma(t)$ & Unit speed geodesic ray \\
    $B^{\gamma}(x)$ & Busemann function associated with the geodesic ray $\gamma$ \\
    $(\calX,\dist)$ & Metric space and its distance function \\
    $\partial \calX$ & Boundary at infinity of $\calX$ \\
    $[x,y]$ & Geodesic segment joining $x$ and $y$ \\
    $\Delta(x,y,z)$ & Geodesic triangle in $\calX$ \\
    $(M^{n}_{K},\dist_{K})$ & Model space of constant curvature $K$ with distance $\dist_{K}$ \\
    $\CAT(K)$ & $\CAT(K)$ space \\
    $HB^{\gamma}_{\tau},\ H^{\gamma}_{\tau}$ & Horoball and horosphere of $\gamma$ at level $\tau$ \\
    \midrule
    $\bbR{n}$ & Euclidean space of dimension $n$ \\
    $\inner{\cdot}{\cdot}$, $\norm{\cdot}$ & Euclidean inner product and norm \\
    $\Rzero$ & Zero vector in $\bbR{n}$ \\
    $\hyperspace{n}$ & Hyperbolic space, either $\pball{n}$ or $\lorentz{n}$ \\
    $K<0$ & Constant sectional curvature \\
    $\Hoplus$ and $\Hodot$ & Gyroaddition and scalar gyromultiplication on $\hyperspace{n}$ \\
    $e$ & Origin in $\hyperspace{n}$ \\
    $\pball{n}$ & Poincar\'e ball model in $\bbR{n}$ with curvature $K$ \\
    $\lambda_x^{K}$ & Conformal factor $\frac{2}{1+K\norm{x}^{2}}$ \\
    $\Moplus,\ \Modot$ & M\"obius gyroaddition and scalar gyromultiplication on $\pball{n}$ \\
    $\lorentz{n}$ & Lorentz model in $\bbR{n+1}$ with curvature $K$ \\
    $\Linner{x}{y}$, $\Lnorm{x}$ & Lorentzian inner product and norm \\
    $x=\left[x_t, x_s^{\top}\right]^{\top}$ & Temporal and spatial decomposition in the Lorentz ambient space \\
    $\Lzero$ & Origin of $\lorentz{n}$ \\
    $\Loplus,\ \Lodot$ & Gyroaddition and scalar gyromultiplication on $\lorentz{n}$ \\
    $\unitsphere{n-1}$ & Unit sphere in $\bbR{n}$ \\
    $B^{v}(x)$ & Busemann function associated with a unit direction $v \in \unitsphere{n-1}$ \\
    $H_{v,\alpha,b}$ & Horosphere $\left\{ x \in \hyperspace{n} \mid -\alpha B^v(x) + b = 0 \right\}$ \\
    $\bar{\dist}(y,H)$ & Signed point-to-hyperplane distance \\
    \midrule
    $\softmax$ & Softmax operator \\
    $u_k(x)$ & Logit for class $k$ \\
    $e_k \in \bbR{m}$ & $k$-th standard basis vector \\
    $\phi$ & Activation function $\bbRscalar \to \bbRscalar$ \\
    \bottomrule
    \end{tabular}
\end{table}

\section{Preliminaries}
\label{app:sec:preliminaries}

\subsection{Riemannian geometry}
\label{app:subsec:riemannian-geometry}
We briefly review Riemannian geometry. For in-depth discussions, please refer to \cite{do1992riemannian}.

\textbf{Tangent space.}
Given a smooth manifold $\calM$, the \emph{tangent space} $T_x\calM$ at $x\in\calM$ is a Euclidean space, consisting of velocities of smooth curves through $x$, namely $v\in T_x\calM$ if there exists a smooth $\gamma$ with $\gamma(0)=x$ and $\dot\gamma(0)=v$.

\textbf{Riemannian manifold.}
A \emph{Riemannian manifold} is a smooth manifold $\calM$ endowed with a \emph{Riemannian metric} $g$, that is, for each $x\in\calM$ an inner product $g_x(\cdot,\cdot)$ or $\inner{\cdot}{\cdot}_x$ on the tangent space $T_x\calM$ varying smoothly with $x$. The metric induces the length of a smooth curve $\gamma:[0,1]\to\calM$ as $L(\gamma)=\int_0^1 \sqrt{g_{\gamma(t)}\left(\dot\gamma(t),\dot\gamma(t)\right)} dt$, and the \emph{geodesic distance} $\dist(x,y)$ is the infimum of lengths over all smooth curves joining $x$ and $y$.

\textbf{Geodesic.}
Straight lines generalize to constant-speed curves that are locally length minimizing between points $x, y \in \calM$, known as \emph{geodesics}:
\begin{equation}
    \gamma^* = \arg \min_\gamma L(\gamma)
    \quad \text{subject to } \gamma(0)=x, \gamma(1)=y, \norm{\dot\gamma(t)}_{\gamma(t)}=c>0.
\end{equation}
We focus on unit-speed geodesics, \ie, $c=1$. In hyperbolic space, between any two points there exists a unique geodesic segment that is globally length minimizing.

\textbf{Exponential and Logarithmic Maps.}
For $x \in \calM$ and $v \in T_x\calM$, let $\gamma_{x,v}$ denote the unique geodesic with $\gamma_{x,v}(0) = x$ and $\dot\gamma_{x,v}(0) = v$. 
The exponential map $\rieexp_x : T_x\calM \supset \mathcal{V} \to \calM$ is defined by $\rieexp_x(v) = \gamma_{x,v}(1)$, where $\mathcal{V}$ is an open neighborhood of the origin in $T_x\calM$. Its local inverse, defined for $y$ in a neighborhood $\mathcal{U} \subset \calM$ of $x$, is the logarithmic map $\rielog_x : \mathcal{U} \to T_x\calM$, satisfying $\rieexp_x \circ \rielog_x = \id_{\mathcal{U}}$. In hyperbolic space, these maps are globally well-defined, that is, $\rieexp_x$ is defined on all of $T_x\calM$ and $\rielog_x(y)$ exists for every $y\in\calM$.

\textbf{Parallel transport.}
\emph{Parallel transport} moves tangent vectors along a curve while preserving the norm. Given a geodesic $\gamma$ from $x$ to $y$, the parallel transport of a tangent vector $v \in T_x\calM$ is the unique vector $\pt{x}{y} (v) \in T_y\calM$ obtained by transporting $v$ along $\gamma$ so that its covariant derivative along $\gamma$ vanishes. Parallel transport defines a linear isometry between $T_x\calM$ and $T_y\calM$.

\subsection{Metric geometry}
\label{app:subsec:metric-geometry}
We present a concise overview of metric geometry, which generalizes Riemannian concepts to metric spaces. The theory develops geodesics and curvature without smooth structure, providing the tools used in our analysis of point-to-horosphere distances. We follow \citep[Ch.~I.1, I.2, II.1, II.2 and II.8]{bridson2013metric} for definitions and results.

\subsubsection{Geodesic metric spaces}

We begin by recalling some basic notions in metric spaces.
\begin{definition}[Metric space]
A \emph{metric space} is a pair $(\calX,\dist)$ where $\calX$ is a set and the distance function $\dist: \calX \times \calX \to \bbRscalar$ satisfies, for all $x,y,z\in\calX$,
\begin{align}
    \text{Positivity: } &\dist(x,y) \ge 0,\quad \dist(x,y)=0 \iff x=y,\\
    \text{Symmetry: }  &\dist(x,y)=\dist(y,x),\\
    \text{Triangle inequality: } &\dist(x,z) \le \dist(x,y)+\dist(y,z)
\end{align}
\end{definition}

Geodesics, rays, and lines generalize unit-speed minimizing geodesics to metric spaces. 
\begin{definition}[Geodesic, geodesic ray and line]
Let $(\calX,\dist)$ be a metric space and let $I=[0,l] \subseteq \bbRscalar$ be a closed interval. A \emph{geodesic} joining $x$ to $y$ is a map $\gamma:I \to \calX$ with $\gamma(0)=x$, $\gamma(l)=y$ such that
\begin{equation}
\dist\left(\gamma(t),\gamma(t')\right)=|t-t'| \quad \text{for all } t,t'\in I\cap[0,l].
\end{equation}
A \emph{geodesic ray} is a map $\gamma:[0,\infty)\to\calX$ such that $\dist\left(\gamma(t),\gamma(t')\right)=|t-t'|$ for all $t,t'\ge 0$. A \emph{geodesic line} is a map $\gamma:\bbRscalar\to\calX$ such that $\dist\left(\gamma(t),\gamma(t')\right)=|t-t'|$ for all $t,t'\in\bbRscalar$.
\end{definition}

Geodesic metric spaces extend the Riemannian premise that any two points can be joined by a geodesic to metric spaces.
\begin{definition}[Geodesic metric space]
The metric space $(\calX,\dist)$ is a \emph{geodesic metric space} (or, more briefly, a \emph{geodesic space}) if every two points are joined by a geodesic. We say that $(\calX,\dist)$ is \emph{uniquely geodesic} if there is exactly one geodesic joining $x$ to $y$ for all $x,y \in \calX$.
\end{definition}

Convexity is defined through geodesics, generalizing linear convexity in Euclidean space and geodesic convexity on manifolds.
\begin{definition}[Convex subset]
Let $(\calX,\dist)$ be a metric space. A subset $C\subseteq\calX$ is \emph{convex} if every pair $x,y\in C$ can be joined by a geodesic in $\calX$ and the image of every such geodesic is contained in $C$.
\end{definition}

\subsubsection{CAT(0) spaces}
We next review concepts that extend nonpositive curvature from manifolds to metric spaces. As a starting point, we recall the model spaces, that is, manifolds of constant curvature, including hyperbolic, Euclidean, and spherical geometries. These serve as reference spaces for metric spaces.
\begin{definition}[Model space $M^{n}_{K}$]
For $K\in\bbRscalar$, the \emph{model space} $(M^{n}_{K},\dist_{K})$ is given by
\begin{equation}
(M^{n}_{K},\dist_{K})=
\begin{cases}
(\unitlorentz{n},\ 1/\sqrt{-K} \dist), & K>0,\\
(\bbR{n},\ \dist ), & K=0,\\
(\unitsphere{n}, 1/\sqrt{K}\dist ), & K<0,
\end{cases}
\end{equation}
where $\unitlorentz{n}$ and $\unitsphere{n}$ are the $n$-dimensional unit Lorentz and sphere manifolds, respectively, and $\dist$ is the geodesic distance in the corresponding manifold. The diameter of $M_k^2$ is denoted $D_k$, which is equal to $\pi / \sqrt{K}$ if $K>0$, and $\infty$ otherwise.
\end{definition}

\begin{definition}[Comparison triangle]
Let $\triangle(x,y,z)$ be a geodesic triangle in $\calX$ with side lengths $a=\dist(y,z)$, $b=\dist(x,z)$, and $c=\dist(x,y)$. A \emph{comparison triangle} for $\triangle(x,y,z)$ is a triangle $\triangle(\bar x,\bar y,\bar z)$ in $M^{2}_{K}$ with $\dist_{K}(\bar y,\bar z)=a$, $\dist_{K}(\bar x,\bar z)=b$, and $\dist_{K}(\bar x,\bar y)=c$. When $a+b+c<2 D_{K}$, the comparison triangle exists.
\end{definition}

CAT($K$) spaces encode curvature through triangle comparison with the model plane $M^2_{K}$. Intuitively, a $\CAT(K)$ space is a metric space where triangles are "thinner" than the corresponding comparison triangles in the model space $M_K^2$.
\begin{definition}[$\CAT(K)$ space]
Let $\calX$ be a metric space and let $K$ be a real number. Let $\Delta$ be a geodesic triangle in $\calX$ with perimeter less than $2 D_K$. Let $\bar{\Delta} \subset M_K^2$ be a comparison triangle for $\Delta$. Then, $\Delta$ is said to satisfy the \emph{$\CAT(K)$ inequality} if for all $x, y \in \Delta$ and all comparison points $\bar{x}, \bar{y} \in \bar{\Delta}$,
\begin{equation}
    d(x, y) \leq d_{K}(\bar{x}, \bar{y}).
\end{equation}
Then, the $\CAT(K)$ space is defined as follows.
\begin{itemize}
    \item
    If $K \leq 0$, then $\calX$ is called a $\CAT(K)$ space (more briefly, "$\calX$ is $\CAT(K)$") if $\calX$ is a geodesic space all of whose geodesic triangles satisfy the $\CAT(K)$ inequality.
    \item 
    If $K>0$, then $\calX$ is called a $\CAT(K)$ space if $\calX$ is $D_K$-geodesic and all geodesic triangles in $\calX$ of perimeter less than $2 D_K$ satisfy the $\CAT(K)$ inequality. 
\end{itemize}
Here, $D_K$-geodesic means that for every pair of points $x, y \in \calX$ with $d(x, y)<D_K$ there is a geodesic joining $x$ to $y$.
\end{definition}

\begin{definition}[Hadamard space]\label{app:def:hadamard}
A \emph{Hadamard space} is a complete $\CAT(0)$ space.
\end{definition}

\begin{proposition}[Orthogonal projection] 
    Let $(\calX,\dist)$ be a $\CAT(0)$ space and let $C\subseteq\calX$ be a convex subset that is complete in the induced metric. For every $x\in\calX$, there exists a unique point $\pi_{C}(x)\in C$ such that
    \begin{equation}
    \dist\left(x,\pi_{C}(x)\right)
    = \inf_{y\in C} \dist\left(x,y\right)=\dist\left(x,C\right).
    \end{equation}
    If $x'$ belongs to the geodesic segment $\left[x,\pi_{C}(x)\right]$, then
    \begin{equation}
    \pi_{C}(x') = \pi_{C}(x).
    \end{equation}
    The map $\pi_{C}: \calX \to C$ is a called an \emph{orthogonal projection}, or simply a \emph{projection}.
\end{proposition}

With geodesics, we can extend the asymptote, Busemann function and horosphere in \cref{sec:preliminaries} to metric spaces.

\begin{definition}[Asymptote]
Let $(\calX,\dist)$ be a metric space. Two geodesic rays $\gamma,\eta:[0,\infty)\to\calX$ are \emph{asymptotic} if
\begin{equation}
\sup_{t\ge 0}  \dist\left(\gamma(t),\eta(t)\right) < \infty.
\end{equation}
The set $\partial \calX$ of \emph{boundary} points of $\calX$, which we shall also call the \emph{points at infinity} or \emph{ideal points}, is the set of equivalence classes of geodesic rays: two geodesic rays being equivalent if and only if they are asymptotic. 
\end{definition}

\begin{definition}[Busemann function, horoball, horosphere]
Let $(\calX,\dist)$ be a metric space and let $\gamma:[0,\infty)\to\calX$ be a geodesic ray. The \emph{Busemann function} associated with $\gamma$ is
\begin{equation}
B^{\gamma}(x) = \lim_{t\to\infty} \left(\dist\left(x,\gamma(t)\right) - t\right),\quad x\in\calX.
\end{equation}
Its sublevel sets $HB^{\gamma}_{\tau}=\left\{x\in\calX: B^{\gamma}(x) \le \tau\right\}$ are \emph{horoballs} and the level sets $H^{\gamma}_{\tau}=\left\{x\in\calX: B^{\gamma}(x)=\tau\right\}$ are \emph{horospheres}.
\end{definition}
As shown by \citet[Lem. II. 8.18]{bridson2013metric}, the limits in the Busemann function exist. Besides, Busemann functions in Hadamard spaces are invariant to the choice of asymptotic geodesic ray \citet[Cor. II. 8.20]{bridson2013metric}.

\begin{corollary}
    If $\calX$ is a Hadamard space, then the Busemann functions associated to asymptotic rays in $\calX$ are equal up to addition of a constant.
\end{corollary}

\subsection{Gyrovector space}

In this subsection we first present the general definitions of gyrogroup and gyrovector space. We then instantiate these notions on concrete manifolds.

\subsubsection{General definition}

Classical vector spaces can be characterized as a commutative group together with a compatible scalar multiplication. By analogy, a gyrovector space is built from a gyrocommutative gyrogroup endowed with a compatible scalar gyromultiplication \citep{ungar2022analytic}.

\begin{definition}[Gyrogroup \citep{ungar2022analytic}]
    \label{def:gyrogroups}
    Given a nonempty set $G$ with a binary operation $\oplus: G \times G \to G$, the pair $(G,\oplus)$ is a \emph{gyrogroup} if, for all $x,y,z\in G$, the following axioms hold:
    
    \noindent (G1) There is at least one element $e \in G$ called a left identity (or neutral element) such that $e \oplus x = x$.
    
    \noindent (G2) There is an element $\ominus x \in G$ called a left inverse of $x$ such that $\ominus x \oplus x = e$.
    
    \noindent (G3) There is an automorphism $\gyr[x, y]: G \rightarrow G$ for each $x, y \in G$ such that
    \begin{equation*}
        x \oplus \left(y \oplus z\right) = \left(x \oplus y\right) \oplus \gyr[x,y]  z \quad \text{(Left gyroassociative law).}
    \end{equation*}
    The map $\gyr[x,y]$ is the \emph{gyration} of $G$ generated by $x$ and $y$.
    
    \noindent (G4) Left reduction law: $\gyr[x, y]=\gyr[x \oplus y, y]$.
\end{definition}

\begin{definition}[Gyrocommutative gyrogroup \citep{ungar2022analytic}]
    \label{def:gyrocommutative_gyrogroups}
    A gyrogroup $(G,\oplus)$ is \emph{gyrocommutative} if
    \begin{equation*}
        x \oplus y = \gyr[x,y] \left(y \oplus x\right) \quad \text{(Gyrocommutative law).}
    \end{equation*}
\end{definition}

\begin{definition}[Gyrovector space \citep{chen2025gyrobnextension}]
\label{def:gyrovector_spaces}
A gyrocommutative gyrogroup $(G,\oplus)$ equipped with a scalar gyromultiplication $\odot: \mathbb{R} \times G \to G$ is a \emph{gyrovector space} if, for $s,t \in \mathbb{R}$ and $x,y,z \in G$, the following axioms hold:
\\ \noindent (V1) Identity scalar multiplication:
$1 \odot x = x$.
\\ \noindent (V2) Scalar distributive law:
$(s+t) \odot x = s \odot x \oplus t \odot x$.
\\ \noindent (V3) Scalar associative law:
$(st) \odot x = s \odot \left(t \odot x\right)$.
\\ \noindent (V4) Gyroautomorphism homogeneity:
$\gyr[x,y] \left(t \odot z\right) = t \odot \gyr[x,y]  z$.
\\ \noindent (V5) Identity gyroautomorphism:
$\gyr\left[s \odot x, t \odot x\right] = \id$, where $\id$ denotes the identity map.
\end{definition}

\textbf{Intuition.} Gyrogroups generalize groups: they are nonassociative, yet obey a controlled form of associativity governed by gyrations. Since gyrations in a group are the identity, every group is a gyrogroup. In the same spirit, gyrovector spaces extend vector spaces and provide an algebraic toolkit that has proved effective for modeling hyperbolic geometry \citep{ungar2022analytic}.

\subsection{Hyperbolic geometry}
\label{app:subsec:hyperbolic-geometry}

\begin{table}[t]
\centering
\renewcommand{\arraystretch}{1.25}
\caption{Riemannian operators on Poincaré ball and Lorentz ($K<0$).}
\label{tab:hyperbolic-operators}
\resizebox{\linewidth}{!}{
\begin{tabular}{lcc}
\toprule
Operator & Poincaré ball $\pball{n}$ & Lorentz $\lorentz{n}$ \\ 
\midrule
Definition &
$\displaystyle \pball{n} = \{x \in \bbR{n} \mid \norm{x}^2 < -1/K \}$ &
$\displaystyle \lorentz{n} = \{x \in \bbR{n+1} \mid \Linner{x}{x} = 1/K,  x_t > 0\}$ \\[6pt]

$g_x(w,v)$ &
$\displaystyle \left(\lambda_x^{K}\right)^2 \inner{w}{v}, \quad \lambda_x^K=\frac{2}{1+K\norm{x}^2}$ &
$\displaystyle \Linner{w}{v} = \inner{v_s}{w_s}- v_t w_t$ \\

$\dist(x,y)$ &
$\displaystyle \frac{2}{\sqrt{|K|}} \tanh^{-1}  \left(\sqrt{|K|}  \norm{-x \Moplus y}\right)$ &
$\displaystyle \frac{1}{\sqrt{|K|}}\cosh^{-1}    \left(K\Linner{x}{y}\right)$ \\

$\rielog_x y$ &
$\displaystyle \frac{2}{\sqrt{|K|}\lambda_x^K}\tanh^{-1}    \left(\sqrt{|K|}  \norm{-x \Moplus y}\right)\frac{-x \Moplus y}{\norm{-x \Moplus y}}$ &
$\displaystyle \frac{\cosh^{-1}(\beta)}{\sqrt{\beta^2 - 1}}(y-\beta x), \quad \beta=K\Linner{x}{y}$ \\

$\rieexp_x v$ &
$\displaystyle x \Moplus  \left(\tanh  \left(\sqrt{|K|}  \frac{\lambda_x^K\norm{v}}{2}\right)\frac{v}{\sqrt{|K|}\norm{v}}\right)$ &
$\displaystyle \cosh(\alpha)x+\frac{\sinh(\alpha)}{\alpha}v, \quad \alpha=\sqrt{|K|}  \Lnorm{v}$ \\

$\pt{x}{y}(v)$ &
$\displaystyle \frac{\lambda_x^K}{\lambda_y^K}  \gyr[y,-x]v$ &
$\displaystyle v-\frac{K\Linner{y}{v}}{1+K\Linner{x}{y}}(x+y)$ \\
\bottomrule
\end{tabular}
}
\end{table}

Given $x,y \in \hyperspace{n}$ and tangent vectors $v,w \in T_x\hyperspace{n}$, \cref{tab:hyperbolic-operators} summarizes the Riemannian operators. 

\textbf{Gyrovector operators.}
The gyro-structure over the hyperbolic space can be defined by its Riemannian operators \cite{ganea2018hyperbolic,chen2025gyrobnextension}. Given $x,y,z \in \hyperspace{n}$ and $t \in \bbRscalar$, the gyroaddition and gyromultiplication are defined as
\begin{align}
x \Hoplus y &= \rieexp_{x} \left(\pt{e}{x}\left(\rielog_{e} y\right)\right),\\
t \Hodot x &= \rieexp_{e}\left(t \rielog_{e} x\right), \\
\gyr[x, y] z &= \Hominus \left(x \Hoplus y\right) \Hoplus \left(x \Hoplus \left(y \Hoplus z\right)\right),
\end{align}
where $e$ denotes the origin in $\hyperspace{n}$. On the Poincaré ball $\pball{n}$, such gyro-structure is known as the Möbius gyrovector space \citep[Ch. 6.14]{ungar2022analytic}:
\begin{align*}
x \Moplus y &= \frac{\left(1-2 K\langle x, y\rangle-K\|y\|^2\right) x+\left(1+K \|x\|^2 \right) y}{1 -2 K\langle x, y\rangle+K^2 \|x\|^2\|y\|^2}, \\
t \Modot x &= \displaystyle \frac{\tanh \left(t \tanh^{-1}(\sqrt{|K|}\|x\|)\right)}{\sqrt{|K|}} \frac{x}{\|x\|}
\end{align*}
where $\Mominus x=-1 \Modot x= -x$ is the gyroinverse and $\Rzero$ is the gyro identity: $\Rzero \Moplus x =x, \forall x \in \pball{n}$. As shown by \citet[Props. 24-25]{chen2025gyrobnextension}, the Lorentz gyroaddition and gyromultiplication also admit closed-form expressions:
\begin{align}
    \label{eq:gyroadd_calmk}
    &x\Loplus y = 
    \begin{cases}
    x, & y=\Lzero, \\
    y, & x=\Lzero, \\
    \begin{bmatrix}
    \frac{1}{\sqrt{|K|}} \frac{D - K N}{D + K N} \\
    \frac{2 \left( A_s x_s + A_y y_s \right)}{ D + K N }
    \end{bmatrix},  & \text{Otherwise}.
    \end{cases} \\
    \label{eq:gyro_prod_calmk}
    &t \Lodot x =
    \begin{cases}
    \Lzero, & t = 0 \lor x=\Lzero \\
    \frac{1}{\sqrt{|K|}} 
    \begin{bmatrix}
    \cosh \left( t \theta \right) \\
    \dfrac{\sinh \left( t \theta \right)}{\norm{x_s}}    x_s
    \end{bmatrix}, & \text{Otherwise},
    \end{cases}
\end{align}
Here, $\theta = \cosh^{-1}(\sqrt{|K|}    x_t)$, $A_s = ab^2 - 2K b s_{xy} - K a n_y$ and $A_y = b(a^2 + K n_x)$ with the following notation:
\begin{equation}
    \begin{aligned}
    & a=1+\sqrt{|K|}x_t,
    b=1+\sqrt{|K|} y_t, \\
    & n_x=\norm{x_s}^2,
    n_y=\norm{y_s}^2, 
    s_{xy}=\langle x_s,y_s\rangle, \\
    &D = a^2b^2 - 2K ab s_{xy} + K^2 n_x n_y, \\
    &N = a^2 n_y + 2ab s_{xy} + b^2 n_x .
    \end{aligned}
\end{equation}
In particular, the gyro identity is $\Lzero$ and the gyroinverse is $\Lominus x 
= -1 \Lodot x = [x_t, -x_s^\top]^\top$.

\section{Comparison with existing hyperbolic MLR}
\label{app:sec:mlr-comparison}

\begin{table}[t]
    \centering
    \caption{Comparison of hyperplanes. Compact params indicates whether the parameterization requires an additional manifold-valued point.}
    \label{app:tab:hyperplane-comparison}
    \begin{tabular}{ccccc}
        \toprule
        Method & Hyperplane & Formulation & \makecell{Applied \\ manifolds} & \makecell{Compact \\ params} \\
        \midrule
        Euclidean MLR & Euclidean & \makecell{$\left\{ x \in \bbR{n} \mid \inner{a}{x} + b = 0 \right\}$ \\ $a \in \bbR{n},\ b \in \bbRscalar$} & $\bbR{n}$ & \cmark \\
        \midrule
        Poincaré MLR \citep{ganea2018hyperbolic} & \makecell{Geodesic \\ \citep[Def. 3.1]{ganea2018hyperbolic}} & \makecell{$\left\{ x \in \pball{n} \mid \inner{ \rieLog{p}\left(x\right)}{a }_{p} = 0 \right\}$ \\ $ p \in \pball{n},\ a \in T_{p}\pball{n}$} & $\pball{n}$ & \xmark \\
        \midrule
        Pseudo-Busemann MLR \citep{nguyen2025neural} & \makecell{Busemann \& gyro \\ \citep[Def. 4.1]{nguyen2025neural}} & \makecell{$\left\{ x \in \pball{n} \mid B^{v}\left( - p \Moplus x \right) = 0 \right\}$ \\ $v \in \unitsphere{n-1},\ p \in \pball{n}$} & $\pball{n}$ & \xmark\\
        \midrule
        Lorentz MLR \citep{bdeir2024fully} & \makecell{Ambient Minkowski \\ \citep[Eq. (7)]{bdeir2024fully}} & \makecell{$\left\{ x \in \lorentz{n} \mid \Linner{w}{x} = 0 \right\}$ \\ $p \in \lorentz{n},\ w \in T_{p}\lorentz{n}$} & $\lorentz{n}$ & \xmark\\
        \midrule
        \rowcolor{HilightColor} BMLR & \Gape[0pt][2pt]{Horosphere} & \Gape[0pt][2pt]{\makecell{$\left\{ x \in \hyperspace{n} \mid -\alpha B^{v}\left(x\right) + b = 0 \right\}$ \\ $\alpha > 0,\ v \in \unitsphere{n-1},\ b \in \bbRscalar$}} & $\pball{n}$, $\lorentz{n}$ & \cmark\\  
        \bottomrule
    \end{tabular}
\end{table}

\begin{table}[t]
    \centering
    \caption{Comparison of point-to-hyperplane distances. \greenbf{Real} means the point-to-hyperplane distance is the real distance, obtained by $\inf_{y \in H} \dist(x,y)$ with $H$ as a hyperplane and $\dist$ as the geodesic distance. Instead, \redbf{Pseudo} means the point-to-hyperplane distance is a surrogate, which only equals to the real distance under the Euclidean geometry.}
    \label{app:tab:distance-comparison}
    \begin{tabular}{cccc}
        \toprule
        Method & Point-to-hyperplane distance & \makecell{Applied \\ manifolds} & Dist \\
        \midrule
        Euclidean MLR & $\displaystyle \frac{|\inner{a}{x} + b|}{\norm{a}}$ & $\bbR{n}$ & \greenbf{Real} \\
        \midrule
        Poincaré MLR \citep{ganea2018hyperbolic} & \makecell{$\displaystyle \frac{1}{\sqrt{-K}} \sinh^{-1}\left(\frac{2\sqrt{-K}  \left|\inner{-p \Moplus x}{a}\right|}{\left(1 + K\norm{-p \Moplus x}^{2}\right)  \norm{a}}\right)$ \\ \citep[Thm. 5]{ganea2018hyperbolic}} & $\pball{n}$ & \greenbf{Real} \\
        \midrule
        Pseudo-Busemann MLR \citep{nguyen2025neural} & \makecell{$\displaystyle \dist(x, p)  \frac{B^{v}\left( -p \Moplus x\right)}{\norm{ -p \Moplus x}}$ \\ \citep[Cor. 4.3]{nguyen2025neural}} & $\pball{n}$ & \redbf{Pseudo} \\
        \midrule
        Lorentz MLR \citep{bdeir2024fully} & \makecell{$\displaystyle \frac{1}{\sqrt{-K}}\left|\sinh^{-1}\left(\sqrt{-K}  \frac{\Linner{v}{x}}{\Lnorm{v}}\right)\right|$ \\ \citep[Eq. (44)]{bdeir2024fully}} & $\lorentz{n}$ & \greenbf{Real} \\
        \midrule
        \rowcolor{HilightColor} BMLR & $\displaystyle \frac{\left|- \alpha B^v(x) + b\right|}{\alpha}$ & $\pball{n}$, $\lorentz{n}$ & \greenbf{Real} \\
        \bottomrule
    \end{tabular}
\end{table}

Hyperbolic MLRs follow a point-to-hyperplane formulation, of which \citet[Eqs. 4-6]{chen2024rmlr} provides the Riemannian prototype. Therefore, the key difference lies in hyperplanes and point-to-hyperplane distances across methods. In addition to \cref{tab:logit-comparison}, \cref{app:tab:hyperplane-comparison,app:tab:distance-comparison} further make this comparison. We draw the following three conclusions.
\begin{enumerate}
    \item \textbf{Hyperplanes.} Our BMLR uses Busemann-based horospheres that simultaneously satisfy three desiderata: 
    (i) compact parameterization without a per-class manifold-valued point, whereas other hyperbolic ones\footnote{We note that \citet{shimizu2021hyperbolic,bdeir2024fully} mitigate this issue through re-parameterization: $a=\pt{e}{p}(z)$ and $p=\exp_e\left(b \frac{z}{\norm{z}}\right)$, where $e \in \hyperspace{n}$ denotes the origin, $z \in \bbR{n}$, and $b \in \bbRscalar$. Nevertheless, the underlying definitions are over-parameterized.} are over-parameterized; 
    (ii) a natural generalization of Euclidean hyperplanes via horospheres, while the Lorentz MLR relies on the ambient Minkowski space, failing to fully respect the intrinsic geometry; 
    and (iii) applicability across hyperbolic models, whereas the Lorentz MLR is tailored to the Lorentz model.
    \item \textbf{Point-to-hyperplane distances.} Although pseudo-Busemann MLR also exploits Busemann functions, it relies on a pseudo point-to-hyperplane distance that only coincides with the real distance in Euclidean geometry. In contrast, our BMLR calculates the real point-to-horosphere distance, ensuring geometric fidelity across hyperbolic models.
    \item \textbf{Batch efficiency.} Recalling \cref{tab:logit-comparison}, the Poincaré MLR \citep{ganea2018hyperbolic} computes logits using $\inner{-p_k \Moplus x}{a_k}$ and $\norm{-p_k \Moplus x}^{2}$. For a batch $X \in \bbR{\mathrm{bs} \times n}$ and $C$ classes, evaluating $-p_k \Moplus X$ for every $k$ yields an intermediate tensor of shape $[\mathrm{bs}, C, n]$, and materializing this tensor can cause GPU out of memory (OOM) when $n$ or $C$ is large. The same limitation holds for the pseudo-Busemann MLR. Consequently, their official implementations compute per class in a for-loop, which is batch inefficient. By contrast, BMLR uses logits $-\alpha_k B^{v_k}(x)+b_k$, whose Busemann term reduces to class-wise inner products $\inner{v_k}{x}$ (or $\inner{v_k}{x_s}$). With $X \in \bbR{\mathrm{bs} \times n}$ and $V=[v_1,\ldots,v_C] \in \bbR{n \times C}$, such inner products can be efficiently implemented as a single matrix multiplication $XV$ without any $[\mathrm{bs}, C, n]$ intermediate, yielding high throughput and low memory usage.
\end{enumerate}

\section{Busemann fully connected layers and point-to-horosphere distances}
\label{app:subsec:busemann-fc-p2h-distance}

An apparently natural attempt to define a hyperbolic FC layer is to replace the LHS of \cref{eq:hyp-fc-p2h} by the \emph{signed point-to-horosphere} distance. The Euclidean hyperplane passing through the origin and orthogonal to $e_k \in \bbR{m}$ is $H_{e_k,0}=\left\{ y \in \bbR{m} \mid \inner{e_k}{y} = 0 \right\}$. The corresponding hyperbolic horosphere, following \cref{eq:horosphere_param}, is $H_{e_k,1,0}=\{y \in \hyperspace{n} \mid -B^{e_k}(y) = 0 \}$. By \cref{eq:point-to-horosphere-dist-param}, the signed point-to-horosphere distance to $H_{e_k,1,0}$ equals $-B^{e_k}(y)$. Accordingly, we can define an alternative FC mapping $\calF: \hyperspace{n} \ni x \mapsto y \in \hyperspace{m}$ via
\begin{equation}\label{eq:busemann-fc-p2horosphere}
    B^{e_k}(y) = \alpha_k B^{v_k}(x) - b_k, \quad k=1,\ldots,m,
\end{equation}
where $u_k(x) = -\alpha_k B^{v_k}(x) + b_k$, and $\alpha_k>0$, $v_k \in \unitsphere{n-1}$, $b_k \in \bbRscalar$ are learnable. Although this only differs from \cref{eq:hyp-fc-p2h} on the LHS, the following discussion shows that such a definition is infeasible in general and fails to deliver a valid hyperbolic FC layer.

\subsection{Poincar\'e model}
Using \cref{eq:poincare-busemann} with $v=e_k$, \cref{eq:busemann-fc-p2horosphere} becomes
\begin{equation}\label{eq:pfc-busemann-eq}
    \frac{1}{\sqrt{-K}} \log\left( \frac{\norm{ e_k - \sqrt{-K} y }^{2}}{1 + K\norm{y}^{2}} \right) = - u_k(x).
\end{equation}
Define $t_k = \exp\left( -\sqrt{-K} u_k(x) \right) > 0$. Exponentiating \cref{eq:pfc-busemann-eq} gives
\begin{equation}\label{eq:pfc-comp-eq}
    t_k \left( 1 + K\norm{y}^{2} \right) = 1 - 2\sqrt{-K} y_k - K\norm{y}^{2}, \quad k=1,\ldots,m.
\end{equation}
Writing $R=\norm{y}^{2}$, \cref{eq:pfc-comp-eq} yields an affine expression for each coordinate
\begin{equation}\label{eq:pfc-affine}
    y_k = c_k + d_k R, \quad c_k = \frac{1 - t_k}{2\sqrt{-K}}, \quad d_k = \frac{\sqrt{-K}}{2}\left(1 + t_k\right).
\end{equation}
Imposing $R = \sum_{k=1}^m y_k^2 = \sum_{k=1}^m\left(c_k + d_k R\right)^2$ gives a quadratic in $R$:
\begin{equation}\label{eq:pfc-R-quadratic}
    A_2 R^2 + \left(A_1 - 1\right) R + A_0 = 0,
\end{equation}
where, denoting $T = \sum_{k=1}^m t_k$ and $q = \sum_{k=1}^m t_k^2$,
\begin{equation}\label{eq:pfc-coeffs}
    A_2 = \frac{-K}{4}\left(m + 2T + q\right), \quad A_1 = \frac{m - q}{2}, \quad A_0 = \frac{m - 2T + q}{4(-K)} \geq 0.
\end{equation}
Its discriminant is
\begin{equation}
    \begin{aligned}
    \Delta_{\mathrm{P}}
    &= (A_1-1)^2 - 4 A_2 A_0\\
&= \left( \frac{m-q}{2} - 1 \right)^{2} - 4 \left( \frac{-K}{4}\left(m+2T+q\right) \right) \left( \frac{m-2T+q}{4(-K)} \right) \\
&= \left( \frac{m-2-q}{2} \right)^{2} - \frac{\left(m+2T+q\right)\left(m-2T+q\right)}{4} \\
&= \frac{1}{4} \left[ (m-2-q)^2 - \left( (m+q)^2 - (2T)^2 \right) \right] \\
&= \frac{1}{4} \left[ \left((m-2-q)-(m+q)\right) \left((m-2-q)+(m+q)\right) + 4T^2 \right] \\
&= \frac{1}{4} \left[ (-2-2q)(2m-2) + 4T^2 \right] \\
    &= T^2 - (m-1)\left(1+q\right).
    \end{aligned}
\end{equation}
A real solution $R$ exists only if $\Delta_{\mathrm{P}} \ge 0$. In addition, feasibility requires $0 \le R < -1/K$ so that $y \in \pball{m}$. Such conditions can fail for generic $\{ u_k(x) \}$, in which case no $y \in \pball{m}$ satisfies \cref{eq:busemann-fc-p2horosphere}.

\subsection{Lorentz model}
Using \cref{eq:lorentz-busemann} with $v=e_k$, \cref{eq:busemann-fc-p2horosphere} becomes
\begin{equation}\label{eq:lfc-busemann-eq}
    \frac{1}{\sqrt{-K}} \log\left( \sqrt{-K} \left( y_t - (y_s)_k \right) \right) = - u_k(x).
\end{equation}
Define $t_k = \exp\left( -\sqrt{-K} u_k(x) \right) > 0$ and, denoting $T = \sum_{k=1}^m t_k$ and $q = \sum_{k=1}^m t_k^2$, we obtain from \cref{eq:lfc-busemann-eq},
\begin{equation}\label{eq:lfc-spatial}
    (y_s)_k = y_t - \frac{t_k}{\sqrt{-K}}, \quad k=1,\ldots,m, \quad\Longrightarrow\quad y_s = y_t\vecone - \frac{1}{\sqrt{-K}} t.
\end{equation}
Enforcing the hyperboloid constraint $\norm{y_s}^{2} - y_t^{2} = 1/K$ yields a quadratic in $y_t$:
\begin{equation}\label{eq:lfc-yt-quadratic}
    (m-1) K y_t^{2} + 2\sqrt{-K} T y_t - \left(1 + q\right) = 0.
\end{equation}
The discriminant is
\begin{equation}\label{eq:lfc-discriminant}
    \begin{aligned}
        \Delta_{\mathrm{L}} 
        &= \left(2\sqrt{-K} T\right)^{2} - 4(m-1)K\left( -\left(1 + q\right) \right) \\
        &= 4(-K) T^{2} + 4(m-1)K\left(1 + q\right) \\
        &= 4(-K) \left[ T^{2} - (m-1)\left(1 + q\right) \right].
    \end{aligned}
\end{equation}
Hence, a real $y_t$ exists only if $T^{2} - (m-1)\left(1 + q\right) \ge 0$. In addition, feasibility requires $y_t>0$. Such conditions can fail for generic $\{ u_k(x) \}$, where no $y \in \lorentz{m}$ satisfies \cref{eq:busemann-fc-p2horosphere}.

\subsection{Summary}
Equating Busemann coordinates as in \cref{eq:busemann-fc-p2horosphere} requires nontrivial inequalities on the responses $\{ u_k(x) \}$. These constraints are not guaranteed during learning, so the system can become infeasible and the output $y$ undefined. This motivates our choice in \cref{eq:hyp-fc-p2h} to use signed point-to-hyperplane distances on the LHS, which admit closed-form solutions that are feasible for all inputs and parameters in both the Poincar\'e and Lorentz models.

\section{Experimental details and additional results}

\subsection{Image classification}
\label{app:subsec:image-classification-details}

\subsubsection{Datasets} 
The CIFAR-10 \citep{krizhevsky2009learning} and CIFAR-100 \citep{krizhevsky2009learning} datasets each contain 60{,}000 $32\times 32$ color images from 10 and 100 classes, respectively. We use the standard PyTorch splits: 50{,}000 training images and 10{,}000 test images. Tiny-ImageNet \citep{le2015tiny} is a subset of ImageNet with 100{,}000 images from 200 classes, resized to $64 \times 64$. We use the official validation split for evaluation. The large-scale ImageNet-1k \citep{deng2009imagenet} dataset contains 1.28M training images, 50K validation images, and 100K test images distributed across 1k classes.

\begin{table}[t]
	\centering
	\caption{Summary of hyperparameters used in the image classification task.}
	\label{tab:hyperparameters_cl}
	\begin{tabular}{cccc}
    \toprule
    Hyperparameter & \makecell{CIFAR-10/100 \\ Tiny-ImageNet}  & ImageNet-1k\\
    \midrule
    {Epochs} & 200  & 100\\
    {Batch size} & 128 & 256\\
    {Initial learning rate} & $0.1$ & $0.1$\\
    {LR schedule} & $60, 120, 160; \gamma=0.2$ & $30, 60, 90; \gamma=0.1$\\
    {Weight decay} & $5e^{-4}$ & $1e^{-4}$\\
    {Optimizer} & SGD & SGD\\
    {Precision} & 32-bit & 32-bit\\
    {\#GPUs} & 1 (RTX A6000) & 2 (RTX A100)\\
    {Curvature $K$} & $-1$ & $-1$\\
    \bottomrule
	\end{tabular}
\end{table}

\subsubsection{Implementation details} 
For CIFAR-10/100 and Tiny-ImageNet, we follow \citet[App. C.1]{bdeir2024fully}. For ImageNet-1k, we follow \citet[Sec. 4]{guo2022clipped}. \cref{tab:hyperparameters_cl} summarizes the dataset-specific hyperparameters. For the hyperbolic MLR, before mapping into the hyperbolic space, we clip the feature vector by
\begin{equation}\label{app:eq:clip-hyper}
  \operatorname{CLIP}\left(x ; r\right)=\min \left\{1, \frac{r}{\norm{x}}\right\} x
\end{equation}
where $r>0$ is a hyperparameter. The clipped Euclidean embedding is projected via the exponential map to the target hyperbolic space: $\rieexp_e\big(\operatorname{CLIP}(x ; r)\big)$. For the Lorentz model, the clipping parameter is $r=1$ on CIFAR-10/100 and $r=4$ on Tiny-ImageNet and ImageNet-1k. On the Poincaré ball, $r=1$ on all four datasets. 

All methods are implemented in PyTorch and trained with cross-entropy loss. The results of MLR, PMLR, and LMLR on CIFAR-10/100 and Tiny-ImageNet are copied from \citet[Tab. 1]{bdeir2024fully}, while the ones of PBMLR-P on CIFAR-10/100 are copied from \citet[Tab. 2]{nguyen2025neural}. The remaining results are obtained by our careful implementation.

\subsection{Genome sequence learning}
\label{app:subsec:genome-sequence-learning-details}

\begin{table}[t]
\centering
\caption{Summary statistics for TEB.}
\label{tab:teb}
\begin{tabular}{cccccc}
\toprule
Task & Species & Datasets & Num. classes & Max length & Train / Dev / Test \\
\midrule
\multirow{3}{*}{Retrotransposons} & \multirow{3}{*}{Plant} & LTR Copia & \multirow{3}{*}{2} & 500 & 7666 / 682 / 568 \\
 &  & LINEs &  & 1000 & 22502 / 2030 / 1782 \\
 &  & SINEs &  & 500 & 21152 / 1836 / 1784 \\
\midrule
\multirow{2}{*}{DNA transposons} & \multirow{2}{*}{Plant} & CMC-EnSpm & \multirow{2}{*}{2} & 200 & 19912 / 1872 / 1808 \\
 &  & hAT-Ac &  & 1000 & 17322 / 1822 / 1428 \\
\midrule
\multirow{2}{*}{Pseudogenes} & \multirow{2}{*}{Human} & processed & \multirow{2}{*}{2} & \multirow{2}{*}{1000} & 17956 / 1046 / 1740 \\
 &  & unprocessed &  &  & 12938 / 766 / 884 \\
\bottomrule
\end{tabular}
\end{table}

\begin{table}[t]
  \centering
  \caption{Summary statistics for the adopted GUE datasets.}
  \label{tab:gue}
  \begin{tabular}{cccccc}
  \toprule
  Task & Species & Dataset & Num. classes & Length & Train / Dev / Test \\
  \midrule
  \multirow{3}{*}{Core promoter detection} & \multirow{3}{*}{Human} & tata & \multirow{3}{*}{2} & \multirow{3}{*}{70} & 4904 / 613 / 613 \\
   &  & notata &  &  & 42452 / 5307 / 5307 \\
   &  & all &  &  & 47356 / 5920 / 5920 \\
  \midrule
  \multirow{3}{*}{Promoter detection} & \multirow{3}{*}{Human} & tata & \multirow{3}{*}{2} & \multirow{3}{*}{300} & 4904 / 613 / 613 \\
   &  & notata &  &  & 42452 / 5307 / 5307 \\
   &  & all &  &  & 47356 / 5920 / 5920 \\
  \midrule
  Covid variant classification & Virus & Covid & 9 & 1000 & 77669 / 7000 / 7000 \\
  \midrule
  \multirow{2}{*}{Species classification} & Fungi & Fungi & 25 & 5000 & 8000 / 1000 / 1000 \\
   & Virus & Virus & 20 & 10000 & 4000 / 500 / 500 \\
  \bottomrule
  \end{tabular}
\end{table}

\subsubsection{Datasets}

\textbf{TEB.}
The Transposable Elements Benchmark (TEB) \citep{khan2025hyperbolic} comprises seven binary classification datasets that investigate transposable elements (TEs) across plant and human genomes. The seven datasets are LTR Copia, LINEs, and SINEs for plant retrotransposons; CMC-EnSpm and hAT-Ac for plant DNA transposons; and processed and unprocessed pseudogenes for human pseudogenes. TEB targets a less explored area of genome organization in genomics deep learning and provides a novel resource for benchmarking models. For each dataset, positive examples are sequences spanning annotated elements of interest, and negatives are randomly sampled, non-overlapping genomic segments outside these regions. We adopt chromosome-level training, validation, and test splits, using chromosomes 8 and 9 for validation and test in plant genomes, and chromosomes 20 to 22 and 17 to 19 for validation and test in human genomes, respectively. Summary statistics are provided in \cref{tab:teb}.

\textbf{GUE.}
The Genome Understanding Evaluation (GUE) benchmark \citep{zhou2024dnabert} is a recently published tool that contains seven biologically significant genome analysis tasks that span 28 datasets. GUE prioritizes genomic datasets that are challenging enough to discern differences between models. The datasets contain sequences ranging from 70 to 1000 base pairs in length and originating from yeast, mouse, human, and virus genomes. In our experiments, we select Core promoter detection and Promoter detection, and the multi-class tasks Covid variant classification and species classification, totaling nine datasets from GUE. Summary statistics are provided in \cref{tab:gue}.

\subsubsection{Implementation details} 
We mainly follow the official implementations of \citet{khan2025hyperbolic} for data processing, model architecture, and training. We adopt a simple CNN with three convolutional blocks followed by dense, ReLU activated layers to extract features \citep[Fig. 4]{khan2025hyperbolic}. Before classification, the features are clipped and mapped to the target hyperbolic space as in \cref{app:eq:clip-hyper}, then passed to either a prior hyperbolic MLR or our BMLR head. The clipping factor defaults to $r=1$, with the following exceptions on GUE: on the Lorentz model, $r=2.0$ for Covid variant classification and $r=5.0$ for species classification; on the Poincar\'e ball, $r=2.0$ for species classification. Following \citet{khan2025hyperbolic}, we treat the curvature $K$ as a learnable parameter initialized as $-1$. The remaining hyperparameters are listed in \cref{tab:genome-hyperparams}. All models share these hyperparameters, except LMLR on Covid variant classification, for which we set the weight decay to $1e^{-3}$ to ensure convergence.

All methods are implemented in PyTorch and trained with cross-entropy loss. Results are obtained from our reimplementation.

\begin{table}[t]
\centering
\caption{Hyperparameters for genome sequence learning.}
\label{tab:genome-hyperparams}
\begin{tabular}{cc}
\toprule
{Batch size} & 100 \\
{Epochs} & 100 \\
{Optimizer} & Adam \\
$\beta_1$, $\beta_2$ & 0.9, 0.999 \\
{Initial learning rate} & $1e^{-4}$\\
{LR schedule} & $60, 85$; $\gamma=0.1$ \\
{Weight decay} & $0.1$ \\
{\#GPUs} & 1 (RTX A6000) \\
{Initial curvature $K$} & $-1$ \\
\bottomrule
\end{tabular}
\end{table}

\subsection{Node classification}
\label{app:subsec:node-classification-details}

\subsubsection{Datasets} 

\begin{table}[t]
  \centering
  \caption{Summary statistics for the node classification datasets.}
  \label{tab:node-dataset-stats}
  \begin{tabular}{ccccc}
  \toprule
  Dataset & \#Nodes & \#Edges & \#Classes & \#Features \\
  \midrule
  Disease & 1044 & 1043 & 2 & 1000 \\
  Airport & 3188 & 18631 & 4 & 4 \\
  PubMed & 19717 & 44338 & 3 & 500 \\
  Cora & 2708 & 5429 & 7 & 1433 \\
  \bottomrule
  \end{tabular}
\end{table}

\textbf{Disease \citep{anderson1991infectious}. }
It represents a disease propagation tree, simulating the SIR disease transmission model, with each node representing either an infection or a non-infection state.

\textbf{Airport \citep{zhang2018link}. } 
It is a transductive dataset where nodes represent airports and edges represent the airline routes as from OpenFlights.org.

\textbf{PubMed \citep{namata2012query}. }
This is a standard benchmark describing citation networks where nodes represent scientific papers in the area of medicine, edges are citations between them, and node labels are academic (sub)areas.

\textbf{Cora \citep{sen2008collective}. }
It is a citation network where nodes represent scientific papers in the area
of machine learning, edges are citations between them, and node labels are academic (sub)areas.

\cref{tab:node-dataset-stats} summarizes the statistics of the datasets.

\subsubsection{Implementation details} 

We follow the official implementations of HGCN \citep{chami2019hyperbolic} and PBMLR \citep{nguyen2025neural} and conduct experiments on the Poincaré ball and the Lorentz model, respectively. We adhere to their experimental settings. The only changes are weight decay and dropout. We train with cross-entropy loss and the Adam optimizer \citep{kingma2015adam} for 5000 epochs with a learning rate of $1e^{-2}$, curvature as $-1$, embedding dimension 16, and three GCN layers. We tune weight decay and dropout and report the values in \cref{tab:node-classification-hyperparams}.

All methods are implemented in PyTorch and run on a single RTX A6000 GPU. The results of HGCN-PMLR and HGCN-PBMLR-P on the Poincaré ball are taken from \citet[Tab. 11]{nguyen2025neural}. Results for the remaining baselines are obtained from our reimplementation following the original settings.

\begin{table}[t]
  \centering
  \caption{Hyperparameters for node classification on Disease, Airport, PubMed, and Cora.}
  \label{tab:node-classification-hyperparams}%
    \begin{tabular}{c|cc}
    \toprule
    Space & Weight decay & Dropout \\
    \midrule
    $\pball{n}$ & $1e-4, 1e-5, 1e-3,1e-3$ & $0.3, 0, 0,0.2$ \\
    $\lorentz{n}$ & $1e-4, 5e-5, 1e-3,1e-3$ & $0, 0, 0,0.3$ \\
    \bottomrule
    \end{tabular}%
\end{table}%

\subsection{Link prediction}
\label{app:subsec:link-prediction-details}
The datasets are identical to those used for node classification. We provide implementation details and additional ablation studies below.

\subsubsection{Implementation details} 
We follow the official implementations of HNN \citep{ganea2018hyperbolic}, HNN++ \citep{shimizu2021hyperbolic}, and HyboNet \citep{chen2022fully}, and adopt the experimental protocol of \citet{chami2019hyperbolic} for link prediction. The encoder consists of two fully connected layers: the first maps the input features to 16, and the second maps 16 to 16. Each FC layer is instantiated as either our BFC or an existing hyperbolic FC layer. After each FC, we apply the activation $\rieexp_e\big(\operatorname{ReLU}(\rielog_e(x))\big)$, where $e$ denotes the model origin. Following \citet{chami2019hyperbolic}, this activation is disabled on Cora. As in the M\"obius layer, we apply a gyro bias after each FC, that is, $x \Hoplus b$. We train with Adam \citep{kingma2015adam} at a learning rate of $1e^{-2}$ and tune weight decay and FC dropout. For BFC, we set $\phi=\tanh$ on Airport and Cora, which yields better performance, while we use the identity map on the other two datasets.

\subsubsection{Ablations on training time and parameter count}
\label{app:subsec:albations-efficiency-bfc}

\begin{table}[t]
  \centering
  \caption{Efficiency comparison: fit time (s/epoch) and parameter count. Slowest results and largest parameter counts are in \redbf{red}.}
   \label{tab:exp-efficiency}%
    \begin{tabular}{c|c|cc|cc|cc|cc}
    \toprule
    \multirow{2}[2]{*}{Space} & \multirow{2}[2]{*}{Method} & \multicolumn{2}{c|}{Disease} & \multicolumn{2}{c|}{Airport} & \multicolumn{2}{c|}{PubMed} & \multicolumn{2}{c}{Cora} \\
    \cmidrule{3-10}          
     &  & Fit Time &  \#Params & Fit Time &  \#Params & Fit Time &  \#Params & Fit Time &  \#Params \\
     \midrule
    \multirow{3}[2]{*}{$\pball{n}$} & M\"obius & 0.0200  & 464   & 0.0535  & 480   & 0.1120  & 8288  & 0.0229  & 23216  \\
     & Poincaré FC & 0.0198  & 528 & 0.0536  & 544   & 0.1176  & 8352  & 0.0248  & 23280  \\
     & \cellcolor{HilightColor}BFC-P & \cellcolor{HilightColor}0.0201  & \cellcolor{HilightColor}528 & \cellcolor{HilightColor}0.0512  & \cellcolor{HilightColor}544   & \cellcolor{HilightColor}0.1123  & \cellcolor{HilightColor}8352  & \cellcolor{HilightColor}0.0231  & \cellcolor{HilightColor}23280  \\
    \midrule
    \multirow{3}[2]{*}{$\lorentz{n}$} & LTFC & \redbf{{0.0343}} & 464   & \redbf{{0.0818}} & 480   & \redbf{{0.1633}} & 8288  & \redbf{{0.0370}} & 23216  \\
     & Lorentz FC & 0.0232  & \redbf{{563}} & 0.0715  & \redbf{{580}} & 0.1537  & \redbf{{8876}} & 0.0261  & \redbf{{24737}} \\
     & \cellcolor{HilightColor}BFC-L & \cellcolor{HilightColor}0.0244  & \cellcolor{HilightColor}528   & \cellcolor{HilightColor}0.0713  & \cellcolor{HilightColor}544   & \cellcolor{HilightColor}0.1525  & \cellcolor{HilightColor}8352  & \cellcolor{HilightColor}0.0280  & \cellcolor{HilightColor}23280  \\
    \bottomrule
    \end{tabular}%
\end{table}%

\cref{tab:exp-efficiency} summarizes fit time per epoch and parameter counts. Our BFC layers achieve training time and model size comparable to existing layers. In particular, LTFC is the slowest due to costly logarithmic and exponential maps, and LFC uses the largest number of parameters among Lorentz variants.

\subsubsection{Ablations on the activation in BFC layers}
\label{app:subsec:albations-activation-bfc}

\begin{table}[htbp]
  \centering
  \caption{Ablation on activation $\phi$ in BFC layers.}
  \label{app:tab:exp-activation}%
    \begin{tabular}{c|c|cc}
    \toprule
    Space & $\phi$ & Airport & Cora \\
    \midrule
    \multirow{2}[2]{*}{$\pball{n}$} & None  & 94.57 ± 0.29 & 87.55 ± 0.34 \\
          & \cellcolor{HilightColor}tanh & \cellcolor{HilightColor}\textbf{94.88 ± 0.39} & \cellcolor{HilightColor}\textbf{91.94 ± 0.32} \\
    \midrule
    \multirow{2}[2]{*}{$\lorentz{n}$} & None  & 94.91 ± 0.24 & 87.45 ± 1.39 \\
          & \cellcolor{HilightColor}tanh & \cellcolor{HilightColor}\textbf{95.37 ± 0.17} & \cellcolor{HilightColor}\textbf{92.28 ± 0.12} \\
    \bottomrule
    \end{tabular}%
\end{table}%

We find that $\phi=\tanh$ in our BFC layers (see \cref{subsec:busemann-fc-generalization}) is beneficial on Airport and Cora. This is validated by \cref{app:tab:exp-activation}, which shows that setting $\phi=\tanh$ consistently improves AUC over the identity activation on the Poincaré and Lorentz models.

\section{Proofs}

\subsection{Proof of \cref{thm:limits-bmlr}}
\label{app:prf:thm:limits-bmlr}
\begin{proof}
Denoting $\kappa^2 = -K$ with $\kappa>0$, we rewrite the Busemann functions in \cref{eq:poincare-busemann,eq:lorentz-busemann} as
\begin{align}
\text{(Poincaré)}\quad &B^{v}(x) = \frac{1}{\kappa} \log\left( \frac{\left\| v - \kappa x \right\|^2}{1 - \kappa^2\left\|x\right\|^2} \right), \\
\text{(Lorentz)}\quad &B^{v}(x) = \frac{1}{\kappa} \log\left( \kappa x_t - \kappa \inner{x_s}{v} \right).
\end{align}
For the Poincaré case,
\begin{align}
\frac{\left\| v - \kappa x \right\|^2}{1 - \kappa^2\left\|x\right\|^2}
= \frac{1 - 2\kappa\inner{v}{x} + \kappa^2\left\|x\right\|^2}{1 - \kappa^2\left\|x\right\|^2}
= 1 + \frac{-2\kappa\inner{v}{x} + 2\kappa^2\left\|x\right\|^2}{1 - \kappa^2\left\|x\right\|^2}.
\end{align}
Using $\log\left(1+u\right)=u + O\left(u^2\right)$ as $u\to 0$ and $\left(1 - \kappa^2\left\|x\right\|^2\right)^{-1}=1 + O\left(\kappa^2\right)$, we obtain
\begin{equation}
\begin{aligned}
B^{v}(x)
&= \frac{1}{\kappa}\left[\left(-2\kappa\inner{v}{x}+2\kappa^2\norm{x}^2\right)\left(1+O\left(\kappa^2\right)\right) + O\left(\kappa^2\right)\right] \\
&= -2\inner{v}{x} + 2\kappa\norm{x}^2 + O\left(\kappa\right), \\
&= -2\inner{v}{x} + O\left(\kappa\right).
\end{aligned}
\end{equation}
Therefore, $B^{v}(x) \xrightarrow{K\to 0^{-}} -2\inner{v}{x}$.

For the Lorentz case, the hyperboloid constraint $-x_t^2 + \left\|x_s\right\|^2 = -\kappa^{-2}$ and $x_t>0$ yield
\begin{equation}
\kappa x_t = \sqrt{1 + \kappa^2 \left\|x_s\right\|^2} = 1 + \frac{1}{2}\kappa^2\left\|x_s\right\|^2 + O\left(\kappa^4\right).
\end{equation}
Set $z = -\kappa\inner{x_s}{v} + \frac{1}{2}\kappa^2\left\|x_s\right\|^2 + O\left(\kappa^4\right)$. Then
\begin{equation}
\log\left( \kappa x_t - \kappa\inner{x_s}{v} \right) = \log\left(1 + z\right) = z + O\left(z^2\right) = -\kappa\inner{x_s}{v} + O\left(\kappa^2\right),
\end{equation}
which implies $B^{v}(x) = -\inner{x_s}{v} + O\left(\kappa\right)$ and therefore $B^{v}(x) \xrightarrow{K\to 0^{-}} -\inner{x_s}{v}$.

Substituting the above two limits into $u_k(x) = -\alpha_k B^{v_k}(x) + b_k$ gives the stated Euclidean limits of the logits.
\end{proof}

\subsection{Proof of \cref{thm:distance-horospheres}}
\label{app:prf:thm:distance-horospheres}

We first review two useful characterizations of Busemann functions in Hadamard spaces.
\begin{definition}
    Let $\mathcal{B}$ be the set of functions 
    $h:\calX \to \bbRscalar$ on $(\calX,\dist)$ satisfying:
    \begin{enumerate}[label=(\roman*)]
        \item $h$ is convex;
        \item 1-Lipschitz: $|h(x)-h(y)|\le \dist(x,y)$ for all $x,y\in \calX$;
        \item for any $x_0\in \calX$ and $r>0$, the function $h$ attains its minimum on the sphere 
        $S_r(x_0)$ at a unique point $y$ with $h(y)=h(x_0)-r$.
    \end{enumerate}
\end{definition}

\begin{proposition}\label{app:prop:characterization-busemann}
For a function $h:\calX \to\bbRscalar$, the following conditions are equivalent:
\begin{enumerate}[label=(\arabic*)]
    \item $h$ is a Busemann;
    \item $h\in\mathcal{B}$;
    \item $h$ is convex, and for every $t\in\bbRscalar$, the set $h^{-1}(-\infty,t]$ is nonempty; moreover, for each $x\in \calX$, the curve $c_x:[0,\infty)\to \calX$ defined by 
    $t\mapsto \pi_{h^{-1}(-\infty,h(x)-t]}(x)$ is a geodesic ray.
\end{enumerate}
\end{proposition}

Now, we are ready to prove \cref{thm:distance-horospheres}.
\begin{proof}
As $\tau_2=\tau_1$ is trivial, we only consider $\tau_2 \neq \tau_1$. We assume $\tau_2 > \tau_1$, and discuss the other direction at last.

\textbf{Step 1: symmetry.}
By definition,
\begin{equation}
\dist \left(H^\gamma_{\tau_1},H^\gamma_{\tau_2}\right)
=\inf_{x\in H^\gamma_{\tau_1},y\in H^\gamma_{\tau_2}}\dist(x,y)
=\inf_{y\in H^\gamma_{\tau_2},x\in H^\gamma_{\tau_1}}\dist(y,x)
=\dist \left(H^\gamma_{\tau_2},H^\gamma_{\tau_1}\right).
\end{equation}

\textbf{Step 2: lower bound.}
For any $x \in H^\gamma_{\tau_2}$ and $y\in H^\gamma_{\tau_1}$, the $1$-Lipschitz property of $B^\gamma$ gives
\begin{equation}
\left|B^\gamma(x)-B^\gamma(y)\right| \le \dist(x,y).
\end{equation}
With $B^\gamma(x)=\tau_2$ and $B^\gamma(y)=\tau_1$ this yields
\begin{equation}\label{eq:lipschitz-lb}
\tau_2-\tau_1 \le \dist(x,y)\qquad\forall y\in H^\gamma_{\tau_1}.
\end{equation}
Taking infimum in \cref{eq:lipschitz-lb} over $y\in H^\gamma_{\tau_1}$ gives
\begin{equation}\label{eq:lower-bound}
\tau_2-\tau_1 \le \dist \left(x,H^\gamma_{\tau_1}\right).
\end{equation}

\textbf{Step 3: upper bound.}
For any $x\in H^\gamma_{\tau_2}$, by property (3) in \cref{app:prop:characterization-busemann},
the projection map
\begin{equation}
c_x(t)=\pi_{\{B^\gamma\le B^\gamma(x)-t\}}(x)
=\pi_{HB^\gamma_{\tau_2 - t}}(x), \quad t \in [0,\infty),
\end{equation}
is a unit-speed geodesic ray: $\dist(x,c_x(t))=t$.

Let $t= \tau_2- \tau_1 > 0$ and $z=c_x(t) \in HB^\gamma_{\tau_1}$. If $B^\gamma(z)<\tau_1$, then $z$ lies in the interior of the horoball $HB^\gamma_{\tau_1}$. We can move slightly from $z$ toward $x$ along the geodesic segment $\overline{xz}$ to obtain a point 
$z_\varepsilon$ with $\dist(x,z_\varepsilon)<\dist(x,z)$, contradicting the minimality of the projection.  
Hence, the projected point $z=c_x(t)$ indeed lies on $H^\gamma_{\tau_1}$: $B^\gamma(z)=\tau_1$. Then, we have the following:
\begin{equation}\label{eq:upper-bound}
\dist(x,H^\gamma_{\tau_1})
\le \dist(x,HB^\gamma_{\tau_1}) = \dist(x,z)
=\dist(x,c_x(t))
=t
=\tau_2-\tau_1.
\end{equation}

\textbf{Step 4: sandwich closure.}
Combining \cref{eq:lower-bound,eq:upper-bound} gives
\begin{equation}
\dist(x,H^\gamma_{\tau_1})=\tau_2-\tau_1.
\end{equation}

Combining \cref{eq:upper-bound} and \cref{eq:lower-bound},
\begin{equation}
\dist \left(x,H^\gamma_{\tau_1}\right) = \tau_2-\tau_1,\qquad\text{for every }x\in H^\gamma_{\tau_2}.
\end{equation}
The right-hand side does not depend on $x$, hence
\begin{equation}
\dist \left(H^\gamma_{\tau_2},H^\gamma_{\tau_1}\right)=\tau_2-\tau_1.
\end{equation}

\textbf{Step 5: opposite direction.}
If $\tau_1>\tau_2$, we can swap the roles of $\tau_1,\tau_2$ above due to the symmetry of the distance between horospheres, which brings
\begin{equation}
\dist \left(H^\gamma_{\tau_2},H^\gamma_{\tau_1}\right)=|\tau_2-\tau_1|.
\end{equation}
\end{proof}

\subsection{Proof of \cref{thm:bfc-poincare}}
\label{app:prf:thm:bfc-poincare}

\begin{proof}
    The hyperplane and point-to-hyperplane distance are presented in \cref{app:tab:hyperplane-comparison,app:tab:distance-comparison}. The origin of the Poincaré ball model is $e = \Rzero \in \pball{m}$. The specific ones w.r.t. the origin \citep[Def. 1 and Eq. (56)]{shimizu2021hyperbolic} are
    \begin{align}
        H_{e_k, \Rzero} 
        &= \{ y \in \pball{m} \mid \inner{e_k}{y} = 0 \}, \\
        \bar{\dist}\left(y, H_{e_k, \Rzero}\right)
        &= \frac{1}{\sqrt{-K}} \sinh^{-1} \left( \frac{2\sqrt{-K} y_k}{1 + K\norm{y}^{2}} \right).
    \end{align}

    Equating $\bar{\dist}\left(y, H_{e_k, \Rzero}\right)$ with $u_k(x)$ from \cref{eq:hyp-fc-p2h} gives
    \begin{equation}\label{eq:poincare-fc-prf-1}
        \sinh^{-1} \left( \frac{2\sqrt{-K} y_k}{1 + K\norm{y}^{2}} \right) = \sqrt{-K} u_k(x), \ \forall 1 \leq k \leq m.
    \end{equation}
    Note that \cref{eq:poincare-fc-prf-1} takes the same form as \citet[Eq. (56)]{shimizu2021hyperbolic}, except that their responses are different. The below proof are inspired by their derivation.

    Applying $\sinh(\cdot)$ on both sides of \cref{eq:poincare-fc-prf-1} yields
    \begin{equation}
        \frac{2\sqrt{-K} y_k}{1 + K\norm{y}^{2}} = \sinh \left( \sqrt{-K} u_k(x) \right).
    \end{equation}
    Define $\omega_k := \frac{1}{\sqrt{-K}} \sinh \left( \sqrt{-K} u_k(x) \right)$ and $\omega = [\omega_k]_{k=1}^{m}$. Then, we have
    \begin{equation}
        2 y_k = \left( 1 + K\norm{y}^{2} \right) \omega_k, \quad \forall k,
    \end{equation}
    which is equivalent to the vector identity
    \begin{equation} \label{eq:poincare-fc-vector-eq}
        2 y = \left( 1 + K\norm{y}^{2} \right) \omega.
    \end{equation}
    Hence, $y$ is collinear with $\omega$. Write $y = \lambda \omega$ with $\lambda \ge 0$. Substituting into \cref{eq:poincare-fc-vector-eq} and taking norms gives a quadratic in $\lambda$:
    \begin{equation}
        K\norm{\omega}^{2} \lambda^{2} - 2\lambda + 1 = 0.
    \end{equation}
    Solving and selecting the branch that satisfies $y \to 0$ as $\omega \to 0$ yields
    \begin{equation}
        \lambda = \frac{1 - \sqrt{1 - K\norm{\omega}^{2}}}{K\norm{\omega}^{2}} = \frac{1}{1 + \sqrt{1 - K\norm{\omega}^{2}}}.
    \end{equation}
    Therefore,
    \begin{equation}
        y = \frac{\omega}{1 + \sqrt{1 - K\norm{\omega}^{2}}}, \quad \omega_k = \frac{\sinh \left( \sqrt{-K} u_k(x) \right)}{\sqrt{-K}},
    \end{equation}
    which proves the claim. One can check that $y \in \pball{m}$.
\end{proof}

\subsection{Proof of \cref{thm:bfc-lorentz}}
\label{app:prf:thm:bfc-lorentz}

\begin{proof}
    Recalling \cref{app:tab:hyperplane-comparison}, a Lorentz hyperplane is 
    \begin{equation}
        H_{w,p} = \{ x \in \lorentz{m} \mid \Linner{w}{x} = 0 \}, \text{ with } p \in \lorentz{m},\ w \in T_{p}\lorentz{m}.
    \end{equation}
    The canonical origin is $\Lzero \in \lorentz{m}$. The tangent space at the origin is
    \begin{equation}
    T_{\Lzero}\lorentz{m}=\{ [0,v^\top]^\top, v \in \bbR{m} \},
    \end{equation}
    where each tangent vector has a zero time component. Therefore, the coordinate hyperplane through the origin and orthogonal to the $k$-th axis is
    \begin{equation}
        \begin{aligned}
            H_{\bar{e}_k,e} &= \{ y \in \lorentz{m} \mid \Linner{\bar{e}_k}{y} = 0 \} \\
            &= \{ y=(y_t,y_s) \in \lorentz{m} \mid (y_s)_k=0 \},
        \end{aligned}
    \end{equation}
    where $\bar{e}_k=[0,e_k^\top]^\top \in T_{\Lzero}\lorentz{m}$. From \cref{app:tab:distance-comparison}, the associated signed point-to-hyperplane distance is
    \begin{equation}
    \begin{aligned}
        \bar{\dist}\left(y, H_{\bar{e}_k,e}\right) 
        &= \sign \left(\Linner{\bar{e}_k}{y}\right) \dist\left(y, H_{\bar{e}_k,e}\right) \\
        &= \frac{1}{\sqrt{-K}} \sinh^{-1} \left( \sqrt{-K} (y_s)_k \right).
    \end{aligned}
    \end{equation}

    Equating $\bar{\dist}\left(y, H_{\bar{e}_k,e}\right)$ with $u_k(x)$ from \cref{eq:hyp-fc-p2h} gives
    \begin{equation}\label{eq:lorentz-fc-prf-1}
        \sinh^{-1} \left( \sqrt{-K} (y_s)_k \right) = \sqrt{-K} u_k(x), \quad 1 \le k \le m.
    \end{equation}
    Applying $\sinh(\cdot)$ to both sides of \cref{eq:lorentz-fc-prf-1} yields
    \begin{equation}
        (y_s)_k = \frac{1}{\sqrt{-K}} \sinh \left( \sqrt{-K} u_k(x) \right), \quad 1 \le k \le m.
    \end{equation}
    Stacking the coordinates gives
    \begin{equation}
        y_s = \frac{1}{\sqrt{-K}} \sinh \left( \sqrt{-K} u(x) \right), \quad u(x)=\left(u_1(x),\ldots,u_m(x)\right)^{\top}.
    \end{equation}

    Since $y \in \lorentz{m}$, the hyperboloid constraint $\Linner{y}{y}=1/K$ implies $-y_t^{2}+\norm{y_s}^{2}=1/K$. Taking the positive time component yields
    \begin{equation}
        y_t = \sqrt{\frac{1}{-K}+\norm{y_s}^{2}}.
    \end{equation}
    Combining the expressions for $y_t$ and $y_s$ proves the claim.
\end{proof}

\subsection{Proof of \cref{thm:limits-bfc}}
\label{app:prf:thm:limits-bfc}

\begin{proof}
Set $K=-\kappa^{2}$ with $\kappa>0$.

\textbf{Poincaré case.}
Recall that
\begin{equation}
    y=\frac{\omega}{1+\sqrt{1 + \kappa^{2}\norm{\omega}^{2}}}, \quad \omega_k=\frac{\sinh\left(\kappa u_k(x)\right)}{\kappa}.
\end{equation}
For any bounded scalar $z$,
\begin{equation}
    \sinh\left(\kappa z\right)=\kappa z+\frac{\kappa^{3} z^{3}}{3!}+O\left(\kappa^{5}\right)
    \ \Rightarrow\ 
    \frac{\sinh\left(\kappa z\right)}{\kappa}=z+\frac{\kappa^{2} z^{3}}{6}+O\left(\kappa^{4}\right).
\end{equation}
Applying this to $z=u_k(x)$ yields
\begin{equation}
    \omega_k=u_k(x)+\frac{\kappa^{2}}{6} u_k(x)^{3}+O\left(\kappa^{4}\right)=u_k(x)+O\left(\kappa^{2}\right).
\end{equation}
Thus, $\omega=u(x)+O\left(\kappa^{2}\right)$ and $\norm{\omega}^{2}=\norm{u(x)}^{2}+O\left(\kappa^{2}\right)$.

For the denominator,
\begin{equation}
    \sqrt{1+\kappa^{2}\norm{\omega}^{2}}=1+\frac{1}{2}\kappa^{2}\norm{\omega}^{2}+O\left(\kappa^{4}\right),
\end{equation}
which gives
\begin{equation}
    1+\sqrt{1+\kappa^{2}\norm{\omega}^{2}}=2+\frac{1}{2}\kappa^{2}\norm{\omega}^{2}+O\left(\kappa^{4}\right).
\end{equation}
Taking the reciprocal produces
\begin{equation}
    \frac{1}{1+\sqrt{1+\kappa^{2}\norm{\omega}^{2}}}=\frac{1}{2}+O\left(\kappa^{2}\right).
\end{equation}

Multiplying with $\omega=u(x)+O\left(\kappa^{2}\right)$ gives
\begin{equation}
    y=\frac{1}{2} u(x)+O\left(\kappa^{2}\right).
\end{equation}
By \cref{thm:limits-bmlr}, $u_k(x)\to 2\alpha_k\inner{v_k}{x}+b_k$, hence
\begin{equation}
    y_k \to \alpha_k\inner{v_k}{x}+\frac{1}{2} b_k.
\end{equation}

\textbf{Lorentz case.}
Recall that
\begin{equation}
    y_s=\frac{1}{\kappa} \sinh\left(\kappa u(x)\right), \quad y_t=\sqrt{\frac{1}{\kappa^{2}}+\norm{y_s}^{2}}.
\end{equation}
Using the same expansion as above,
\begin{equation}
    y_{s,k}=\frac{1}{\kappa} \left(\kappa u_k(x)+\frac{\kappa^{3}}{3!} u_k(x)^{3}+O\left(\kappa^{5}\right)\right)=u_k(x)+O\left(\kappa^{2}\right).
\end{equation}
Therefore, $y_s=u(x)+O\left(\kappa^{2}\right)$ and $\norm{y_s}^{2}=\norm{u(x)}^{2}+O\left(\kappa^{2}\right)$.

Factor out $\kappa^{-1}$ and expand the square root:
\begin{equation}
\begin{aligned}
    y_t&=\frac{1}{\kappa} \sqrt{1+\kappa^{2}\norm{y_s}^{2}} \\
        &=\frac{1}{\kappa} \left(1+\frac{1}{2}\kappa^{2}\norm{y_s}^{2}+O\left(\kappa^{4}\right)\right) \\
        &=\frac{1}{\kappa}+\frac{\kappa}{2}\norm{y_s}^{2}+O\left(\kappa^{3}\right) \\
        &=\frac{1}{\kappa}+O\left(\kappa\right) \to \infty.
\end{aligned}
\end{equation}

By \cref{thm:limits-bmlr}, $u_k(x)\to \alpha_k\inner{v_k}{x_s}+b_k$. Using the spatial expansion yields
\begin{equation}
    (y_s)_k \to \alpha_k\inner{v_k}{x_s}+b_k.
\end{equation}
This completes the proof.
\end{proof}

\end{document}